\pdfoutput=1

\documentclass[11pt]{article}

\usepackage[final]{acl}

\usepackage{times}
\usepackage{latexsym}

\usepackage[T1]{fontenc}

\usepackage[utf8]{inputenc}

\usepackage{microtype}

\usepackage{inconsolata}

\usepackage{graphicx}

\usepackage{amsmath,amssymb,amsfonts}
\usepackage{amsthm}

\newtheorem{theorem}{Theorem}
\newtheorem{assumption}{Assumption}
\newtheorem{lemma}{Lemma}

\DeclareMathOperator*{\argmin}{arg\,min}

\usepackage{booktabs}
\usepackage{multicol}
\usepackage{multirow}
\usepackage{color}
\usepackage{xcolor}
\usepackage{subfigure}

\usepackage{algorithm}
\usepackage{algorithmic}

\usepackage{array} 

\usepackage{tabularx} 
\usepackage{makecell}

\title{

FedProxy: Federated Fine-Tuning of LLMs via  Proxy SLMs and Heterogeneity-Aware Fusion
}

\author{
 \textbf{Tao Fan\textsuperscript{1, 2}},
 \textbf{Guoqiang Ma\textsuperscript{2}},
 \textbf{Yuanfeng Song\textsuperscript{2}},
 \textbf{Lixin Fan\textsuperscript{2}},
  \textbf{Kai Chen\textsuperscript{1}},
 \textbf{Qiang Yang\textsuperscript{3}}
\\
 \textsuperscript{1}Hong Kong University of Science and Technology, Hong Kong, China
 \\
 \textsuperscript{2}WeBank Co., Ltd, Shenzhen, China
 \\
 \textsuperscript{3}Hong Kong Polytechnic University, Hong Kong, China
\\
 \small{
   \textbf{Correspondence:} \href{mailto:tfanac@cse.ust.hk}{tfanac@cse.ust.hk}, \href{mailto:qyang@cse.ust.hk}{qyang@cse.ust.hk}
 }
}

\begin{document}
\maketitle

\begin{abstract}
Federated fine-tuning of Large Language Models (LLMs) is obstructed by a trilemma of challenges: protecting LLMs intellectual property (IP), ensuring client privacy, and mitigating performance loss on heterogeneous data. Existing methods like Offsite-Tuning (OT) secure the LLMs IP by having clients train only lightweight adapters, yet our analysis reveals they suffer from a fundamental performance bottleneck, leaving a significant gap compared to centralized training. To bridge this gap, we introduce FedProxy, a new federated adaptation framework. FedProxy replaces weak adapters with a unified, powerful Proxy Small Language Model (SLM), compressed from the proprietary LLM, to serve as a high-fidelity surrogate for collaborative fine-tuning. Our framework systematically resolves the trilemma through a three-stage architecture: (i) Efficient Representation via server-guided compression to create a resource-friendly proxy; (ii) Robust Optimization through an interference-mitigating aggregation strategy to handle data heterogeneity; and (iii) Effortless Fusion via a training-free "plug-in" mechanism to integrate learned knowledge back into the LLM. Experiments show FedProxy significantly outperforms OT methods and approaches centralized performance, establishing a new benchmark for secure and high-performance federated LLM adaptation.
\end{abstract}

\section{Introduction}

Large Language Models (LLMs)~\cite{OpenAI2023GPT4TR, touvron2023llama2,guo2025deepseek,yang2025qwen3} have revolutionized natural language understanding and generation. However, fine-tuning LLMs on decentralized, privacy sensitive data via Federated Learning (FL)~\cite{mcmahan2017communication,yang2019federated} faces a critical trilemma: (i) preserving proprietary model intellectual property (IP), (ii) safeguarding client data privacy, and (iii) mitigating performance degradation from heterogeneous data distributions~\cite{kang2023grounding, fan2023fate-llm,fan2025fedmkt,fan2025ten}.
Offsite Tuning (OT)~\cite{xiao2023offsite} and its variants~\cite{zhang2023crash, yao2025scaleot, yao2025gradot, kuang2024federatedscope,wu2024fedbiot} emerged to address this by training only lightweight adapters on the client side, keeping the backbone frozen and obfuscated. Despite privacy benefits, our analysis reveals a significant performance bottleneck in OT based frameworks due to inherent representation capacity limitations, hindering practical utility.

To bridge this gap, we introduce \textit{\textbf{FedProxy}}, a holistic framework for secure federated LLM adaptation. FedProxy adopts a surrogate driven optimization paradigm. By distilling the proprietary LLM into a \textit{\textbf{Proxy Small Language Model (SLM)}}, we create a high fidelity semantic vessel retaining backbone structural knowledge, computationally feasible for client-side training.
As illustrated in Figure \ref{fig:fedproxy_overview}, FedProxy addresses the trilemma through a three stage architectural synthesis:

\begin{itemize}
    \item \textbf{Server Guided Compression.} The server compresses the proprietary LLM into a proxy SLM using public data. This protects model IP and accommodates client side resource constraints.
    
    \item \textbf{Interference Mitigating Aggregation.} To address data heterogeneity, we propose a multi-stage aggregation protocol analyzing client heterogeneity and parameter conflicts. This guides conflict-aware local training and heterogeneity-aware server-side merging.
    
    \item \textbf{Training Free Knowledge Fusion.}
    We design a parameter-space "plug in" mechanism directly integrating refined proxy weights into the original LLM. This enables knowledge transfer without costly retraining or additional inference latency.

\end{itemize}

Our key contributions are as follows:

\begin{itemize}

    \item We propose \textit{\textbf{FedProxy}}, a holistic framework resolving the trilemma of IP protection, data privacy, and model performance. It orchestrates a three stage pipeline for effective federated fine-tuning via a proxy SLM.
    
    \item We propose \textbf{\textit{H-TIES}} and \textbf{\textit{PCR}}, a heterogeneity-aware aggregation strategy. By analyzing parameter-level interference, this strategy guides conflict-aware local training and server-side merging, which effectively alleviates parameter interference issues in heterogeneous tasks.
    
    \item Extensive experiments demonstrate FedProxy significantly outperforms OT-based methods and achieves performance comparable to centralized fine-tuning.
\end{itemize}

\begin{figure*}[ht]
 \centering

  \includegraphics[width=0.95\linewidth]{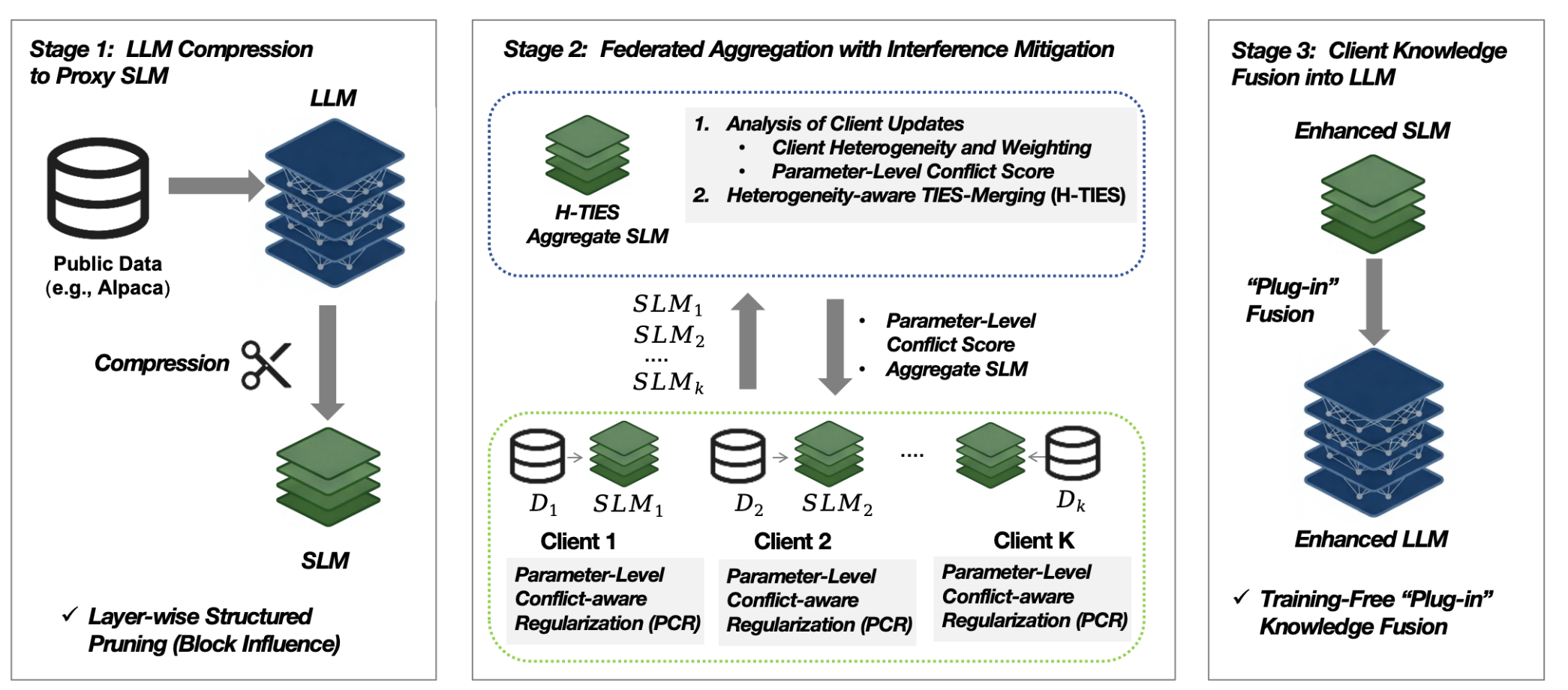}
  
   \caption{Overview of the FedProxy framework. The framework consists of three main stages: (1) LLM Compression to Proxy SLM, (2) Client-side SLM training and federated aggregation with interference mitigation, and (3) Client knowledge fusion back into the global LLM.}
\label{fig:fedproxy_overview}

\end{figure*}

\section{Related Work}

\subsection{Offsite-Tuning}

To address privacy-preserving fine-tuning in client-server architectures where direct exposure of the server's proprietary LLM is undesirable, \textit{Offsite-Tuning} (OT) ~\cite{xiao2023offsite} has been introduced, where the server sends a lightweight, frozen LLM \textit{emulator} and trainable adapter modules to the client, who fine-tunes only the adapters and returns them for integration into the full LLM, preserving both client data privacy and server model IP. Subsequent works have aimed to enhance the OT framework. For instance, CRaSh~\cite{zhang2023crash} improves emulator generation through layer clustering and dropping, while ScaleOT~\cite{yao2025scaleot} and GradOT~\cite{yao2025gradot} introduce compression techniques based on reinforcement learning and gradient preservation. has also been extended to federated learning~\cite{yang2019federated,fan2025ten} with FedOT~\cite{kuang2024federatedscope} and FedBiOT~\cite{wu2024fedbiot}. 
However, our empirical analysis (Section~\ref{sec:experiments}) reveals that OT-based methods consistently fall short of centralized fine-tuning performance.

\subsection{Model Merging}
Model merging seeks to combine multiple fine-tuned models into a single, powerful one without retraining. While naive weight averaging \cite{wortsman2022model} often fails, advanced techniques address task conflicts. Some methods, like Fisher-Merging \cite{matena2022merging} and RegMean \cite{jin2023dataless}, compute task-aware coefficients. Others operate on \textit{task vectors} (the difference between fine-tuned and pretrained weights) \cite{ilharco2023editing}. More sophisticated approaches, such as TIES-Merging \cite{yadav2023ties} and DARE \cite{yu2024language}, resolve interference between these vectors by trimming redundant parameters and rescaling weights, leading to more robust merged models.

\section{Problem Formulation}
\label{sec:problem}
The goal of federated LLM fine-tuning is to enhance a proprietary, server-held LLM $f_\theta$ by leveraging private data $\{\mathcal{D}_k\}_{k=1}^K$ from $K$ clients. This task is fundamentally constrained by the need to protect the LLM's intellectual property (IP), accommodate clients' limited resources, ensure data privacy, and handle statistical data heterogeneity. We decompose this into three core sub-problems:

\textbf{1. Efficient and Secure Model Representation.}
To protect model IP and accommodate client constraints, transmitting the full LLM $f_\theta$ is infeasible. The first problem is to create a compact proxy model $f_\phi$ ($|\phi| \ll |\theta|$) that is powerful enough for fine-tuning yet small enough for efficient on-device training and communication. The objective is to find a compression function $g$ to generate the initial proxy $\phi^{(0)}$:
\begin{equation}
\label{eq:prob_compress}
\phi^{(0)} = g(\theta, \mathcal{D}_{\text{pub}})
\end{equation}

\textbf{2. Heterogeneity-Aware Federated Optimization.}
Due to diverse local data $\mathcal{D}_k$, aggregating client-trained models $\{\phi_k^{(t)}\}$ can lead to "negative interference", where conflicting updates degrade performance. The second problem is to devise a robust federated optimization strategy $\mathcal{A}$ that mitigates this interference. The goal is to produce an improved global proxy model $\phi^{(t+1)}$ at each round:
\begin{equation}
\label{eq:prob_agg}
\phi^{(t+1)} = \mathcal{A}(\{\phi_k^{(t)}\}_{k=1}^K, \phi^{(t)})
\end{equation}

\textbf{3. Client Knowledge Fusion into the LLM.}
Finally, the client knowledge aggregated in the converged proxy model $\phi^*$ must be transferred back to enhance the original LLM $f_\theta$. The third problem is to design a fusion mechanism $\mathcal{F}$ that integrates this specialized knowledge, improving the LLM without direct access to client data:
\begin{equation}
\label{eq:prob_fuse}
\theta_{\text{new}} = \mathcal{F}(\theta_{\text{old}}, \phi^*)
\end{equation}

\section{The Proposed FedProxy Framework}
To address the sub-problems defined in Section~\ref{sec:problem}, we introduce \textbf{FedProxy}, a three-stage framework that provides a concrete solution to the problem decomposition. We illustrate the FedProxy architecture in Figure \ref{fig:fedproxy_overview}, with further details provided in Algorithm \ref{alg:fedproxy}. The first stage, \textbf{Efficient Model Representation via Compression}, is detailed in Section~\ref{ss:llm_compression}. The second stage, \textbf{Federated Aggregation with Interference Mitigation}, is covered in Section~\ref{ss:parameter-interference-mitigation}. Finally, the \textbf{Training-Free "Plug-in" Knowledge Fusion} is explained in Section~\ref{ss:knowledge_fusion}. For theoretical analysis, see Appendix~\ref{sec:appendix-theoretical-analysis}.

\subsection{LLM Compression to Proxy SLM}
\label{ss:llm_compression}

The first phase of the FedProxy framework is to compress a large, general-purpose LLM, denoted as $f_\theta$, into a single, efficient proxy SLM, $f_\phi$. This proxy SLM serves as a shared, high-quality foundation for all clients, significantly reducing the computational and communication overhead required for federated fine-tuning.
Instead of relying on client-specific data, which can be sensitive and heterogeneous, we perform this compression using a broad, publicly available instruction-following dataset, $\mathcal{D}_{\text{pub}}$ (e.g., Alpaca~\cite{alpaca}). This approach ensures that the resulting SLM retains a strong, task-agnostic reasoning and instruction-following capability, making it an effective starting point for diverse downstream tasks.

The compression is achieved through structured pruning, where we identify and remove entire transformer blocks from the original LLM. We quantify the importance of each block using the Block Influence (BI) metric, adapted from ShortGPT~\cite{men2024shortgpt}. The BI score for the $i$-th transformer block is calculated based on its impact on hidden state representations when processing the public dataset:
\begin{equation}
\label{eq:bi_metric}
\text{BI}_i = 1 - \mathbb{E}_{X \sim \mathcal{D}_{\text{pub}}, t} \left[ \frac{X_{i,t}^{\top} X_{i+1,t}}{\|X_{i,t}\|_2 \|X_{i+1,t}\|_2} \right]
\end{equation}
where $X_{i,t}$ denotes the representation of the $t$-th token at the output of block $i$. A higher BI score indicates that a block contributes more significantly to the model's representational capacity.

Based on these scores, we generate a global pruning mask by selecting the blocks with the highest BI scores until the desired compression ratio $\kappa$ is met. This process yields a single, compact \textbf{\textit{proxy SLM}} architecture, $f_\phi$, where $|\phi| \ll |\theta|$. This unified Proxy SLM is then used as the initial model for all clients in the subsequent federated fine-tuning stage (Section~\ref{ss:parameter-interference-mitigation}).

\subsection{Federated Aggregation with Interference Mitigation}
\label{ss:parameter-interference-mitigation}
Parameter interference represents a critical challenge in federated model merging, particularly when aggregating models trained on heterogeneous tasks. As demonstrated in prior research \cite{yu2024language, yadav2023ties}, unweighted averaging can significantly degrade model performance through two primary interference mechanisms: magnitude interference, where task-critical parameters are diluted, and directional interference, where parameters with opposing signs cancel each other out.

To address this, FedProxy implements a three-stage aggregation strategy in each communication round $t$. First, the server analyzes all incoming client updates to compute metrics for heterogeneity, weighting, and parameter-level conflict. Second, clients use these metrics for a conflict-aware regularization during their local training. Finally, the server uses the metrics to perform a heterogeneity-aware model merge. This process is detailed below.

\subsubsection{Server-Side Analysis of Client Updates}
At the end of a training round $t$, after receiving updated models $\phi_k^{(t)}$ from all clients, the server performs a comprehensive analysis. This step generates a set of metrics that will guide both the next round of client training and the current round's aggregation.

\textbf{Task Vector and Similarity Calculation}. The server first computes each client's task vector, $\tau_k^{(t)} = \phi_k^{(t)} - \phi^{(t-1)}$, representing the update from the previous global model. It then calculates pairwise cosine similarity scores $S_{k,j}^{(t)} = \cos(\tau_k^{(t)}, \tau_j^{(t)})$ between all client pairs.

\textbf{Heterogeneity and Weighting Metrics}. Based on the similarity scores $S_{k,j}^{(t)}$, the server derives two key metrics by evaluating the relationship between client $k$ and its peers $j \in \{1, \dots, K\} \setminus \{k\}$:

\begin{itemize}
   \item A \textbf{heterogeneity coefficient} $h_k^{(t)}$, which measures how much a client deviates from the rest of the population. By averaging similarities with all other clients, it identifies unique or outlier task updates.
   \begin{equation}
   \label{eq:heterogeneity}
   h_k^{(t)} = 1 - \frac{1}{K-1} \sum_{j=1, j \neq k}^{K} \max\big(0, S_{k,j}^{(t)}\big)
   \end{equation}
   
   \item An \textbf{adaptive aggregation weight} $w_k^{(t)}$, which prioritizes clients that exhibit strong consensus with the broader network.
   \begin{equation}
   \label{eq:weight}
   w_k^{(t)} = \frac{\exp\left( \sum_{j=1, j \neq k}^{K} |S_{k,j}^{(t)}| \right)}{\sum_{k'=1}^K \exp\left( \sum_{j=1, j \neq k'}^{K} |S_{k',j}^{(t)}| \right)}
   \end{equation}
\end{itemize}
The heterogeneity coefficient is also normalized via $h_{k,\text{norm}}^{(t)} = (h_k^{(t)} - h_{\text{min}}) / (h_{\text{max}} - h_{\text{min}})$. These two metrics, $h_{k,\text{norm}}^{(t)}$ and $w_k^{(t)}$, are used in the server-side H-TIES merging step.

\textbf{Parameter-Level Conflict Score}. Finally, the server calculates a conflict score $C_d^{(t)}$ for each parameter dimension $d$, quantifying the disagreement in update signs across all clients.
 \begin{equation}
 \label{eq:param_conflict}
 C_d^{(t)} = 1 - \frac{1}{K} \left| \sum_{k=1}^K \text{sign}(\tau_k^{(t)}[d]) \right|
 \end{equation}
A score near 1 indicates high conflict (updates pull in different directions), while a score near 0 signifies strong agreement. This conflict score is sent to clients to be used in the next training round.

\subsubsection{Parameter-Level Conflict-Aware Client-Side Regularization (PCR)}
To proactively mitigate interference, clients incorporate the server-provided conflict scores into their local training objective for the next round, $t+1$. The composite loss function $\mathcal{L}_k$ for client $k$ combines the standard task loss with a dynamic regularization term:
\begin{equation}
\label{eq:client_train}
\mathcal{L}_k = \mathcal{L}_{\text{task}}(\phi_k; \mathcal{D}_k) + \lambda_{\text{reg}} \cdot \mathcal{L}_{\text{reg},k}
\end{equation}
where $\phi_k$ is the model being trained, and the regularization term $\mathcal{L}_{\text{reg},k}$ penalizes divergence from the  previous global model $\phi^{(t)}$. The penalty is scaled by the parameter-level conflict scores $C_d^{(t)}$ from the previous round:
\begin{equation}
\label{eq:client_reg_2}
\mathcal{L}_{\text{reg},k} = \sum_d  C_d^{(t)} \cdot \|\phi_k[d] - \phi^{(t)}[d]\|_2^2
\end{equation}
This mechanism encourages clients to align with the global model on consensus parameters (where $C_d^{(t)}$ is high) while allowing more freedom for personalization on conflicting parameters (where $C_d^{(t)}$ is low).

\subsubsection{Heterogeneity-aware TIES-merging (H-TIES)}
To effectively merge client updates, the server employs the \textbf{Heterogeneity-aware TIES-merging (H-TIES)} strategy. This approach adapts the aggregation process by dynamically using the heterogeneity coefficient $h_k^{(t)}$ and aggregation weight $w_k^{(t)}$. It consists of three steps:

\textbf{1. Heterogeneity-Adaptive Sparsification.}
H-TIES adjusts the sparsity of each client's update based on its heterogeneity. The retention rate $r_k^{(t)}$ is inversely proportional to the heterogeneity coefficient $h_{k,\text{norm}}^{(t)}$. Let $r_0$ be the base retention rate and $\delta$ be the adaptive strength:
\begin{equation}
r_k^{(t)} = \max\big(0, \min(1, r_0 - \delta \cdot h_{k,\text{norm}}^{(t)})\big)
\end{equation}
Clients with high heterogeneity (large $h_{k,\text{norm}}^{(t)}$) will have their updates made sparser, while homogeneous clients (small $h_{k,\text{norm}}^{(t)}$) retain a larger portion. This yields the \textbf{\textit{sparsified update}} $\tilde{\tau}_k^{(t)}$.

\textbf{2. Pre-Aggregation Scaling.}
Each client's influence is then scaled by its aggregation weight $w_k^{(t)}$ to amplify the contributions of clients within a strong consensus:
\begin{equation}
\hat{\tau}_k^{(t)} = w_k^{(t)} \tilde{\tau}_k^{(t)}
\end{equation}

\textbf{3. Weighted Conflict Resolution and Merging.}
Finally, the server resolves conflicts and merges the updates using a method inspired by the principles of TIES-merging~\cite{yadav2023ties}. This is a pure merging step without any server-side learning. For each parameter dimension $d$, it first computes the total weighted magnitude of positive ($P_d$) and negative ($N_d$) updates:
\begin{equation}
P_d = \sum_{k: \hat{\tau}_k^{(t)}[d] > 0} |\hat{\tau}_k^{(t)}[d]|, \quad
N_d = \sum_{k: \hat{\tau}_k^{(t)}[d] < 0} |\hat{\tau}_k^{(t)}[d]|
\end{equation}
An update $\Delta \phi^{(t)}[d]$ is then computed by performing a weighted average of only the updates that conform to a dominant sign. A sign is dominant if its total magnitude sufficiently outweighs the other, controlled by a threshold $\rho \ge 1$:
\begin{equation}
\Delta \phi^{(t)}[d] = 
\begin{cases}
\; \frac{\sum_{k: \hat{\tau}_k^{(t)}[d] > 0} \hat{\tau}_k^{(t)}[d]}{\sum_{k: \hat{\tau}_k^{(t)}[d] > 0} w_k}, & \text{if } \frac{P_d}{N_d + \varepsilon} \ge \rho \\
\; \frac{\sum_{k: \hat{\tau}_k^{(t)}[d] < 0} \hat{\tau}_k^{(t)}[d]}{\sum_{k: \hat{\tau}_k^{(t)}[d] < 0} w_k}, & \text{if } \frac{N_d}{P_d + \varepsilon} \ge \rho \\
\; 0, & \text{otherwise}
\end{cases}
\end{equation}
where $\varepsilon$ is a stability constant. This step performs a weighted average of the conforming updates, preventing interference from conflicting signs.

The global \textbf{\textit{proxy SLM }}is then updated by directly applying this aggregated delta:
\begin{equation}
\phi^{(t)} = \phi^{(t-1)} + \Delta \phi^{(t)}
\end{equation}

\subsection{Training-Free "Plug-in" Knowledge Fusion}

\label{ss:knowledge_fusion}

In the final stage, we fuse the knowledge from the aggregated proxy SLM, $f_\phi$, back into the original LLM, $f_\theta$. Our approach is simple, effective, and entirely training-free. Leveraging the architectural mapping where the proxy SLM is a direct sub-network of the LLM (i.e., $\phi \subset \theta$), we perform a direct parameter replacement. Specifically, we take the updated layers from the global proxy SLM and use them to overwrite the corresponding layers in the full LLM. This "plug-in" mechanism seamlessly integrates the federated enhancements without requiring any costly retraining or additional inference latency on the massive LLM, thus efficiently updating its capabilities with the knowledge gained from clients.

\begin{algorithm}[h] 
\footnotesize
\caption{The FedProxy Framework}
\label{alg:fedproxy}

\textbf{Input:} \\
    Proprietary LLM $f_\theta$; Public dataset $\mathcal{D}_{\text{pub}}$; \\
    Number of clients $K$; Communication rounds $T$; \\
    Client datasets $\{\mathcal{D}_k\}_{k=1}^K$; Compression ratio $\kappa$;

\textbf{Output:} Enhanced global LLM $\theta_{\text{final}}$

\begin{algorithmic}[1]
  \STATE \textbf{Server-Side Initialization:}
  \STATE Generate proxy SLM: $\phi^{(0)} \gets \text{Compress}(\theta, \mathcal{D}_{\text{pub}}, \kappa)$ (Sec~\ref{ss:llm_compression})
  \STATE Initialize parameter conflict scores $C_d^{(0)} \gets \mathbf{0}$ for all dimensions $d$

  \FOR{each round $t = 0$ \TO $T-1$}
    \STATE \textbf{Server-Side (Distribution):}
    \STATE Distribute global proxy $\phi^{(t)}$ and conflict scores $\{C_d^{(t)}\}$ to all clients.

    \STATE \textbf{Client-Side Update (Parallel):}
    \FOR{each client $k \in [K]$}
      \STATE Receive $\phi^{(t)}$ and $\{C_d^{(t)}\}$.
      \STATE Train $\phi_k$ on $\mathcal{D}_k$ to minimize $\mathcal{L}_k$ (Eq.~\ref{eq:client_train}) $\rightarrow$ get $\phi_k^{(t+1)}$.
      \STATE Send updated proxy $\phi_k^{(t+1)}$ to server.
    \ENDFOR

    \STATE \textbf{Server-Side (Analysis and Aggregation):}
    \STATE Receive $\{\phi_k^{(t+1)}\}_{k=1}^K$ from clients.
    \STATE \textit{// Server-Side Analysis of Client Updates}
    \STATE Compute task vectors $\tau_k^{(t+1)} \gets \phi_k^{(t+1)} - \phi^{(t)}$.
    \STATE Compute heterogeneity $\{h_k^{(t+1)}\}$ (Eq.~\ref{eq:heterogeneity}) and weights $\{w_k^{(t+1)}\}$ (Eq.~\ref{eq:weight}).
    \STATE Compute parameter conflict scores $\{C_d^{(t+1)}\}$ for the next round (Eq.~\ref{eq:param_conflict}).
    
    \STATE \textit{// Heterogeneity-Aware Aggregation (H-TIES)}
    \STATE Aggregate task vectors $\{\tau_k^{(t+1)}\}$ using H-TIES merging (Sec~\ref{ss:parameter-interference-mitigation}) to get global update $\Delta\phi^{(t+1)}$.
    \STATE Update global proxy: $\phi^{(t+1)} \gets \phi^{(t)} + \Delta\phi^{(t+1)}$.
  \ENDFOR
  
  \STATE \textbf{Final Knowledge Fusion:}
  \STATE Fuse converged proxy $\phi^{(T)}$ into base LLM $\theta$ (Sec~\ref{ss:knowledge_fusion}).
  \STATE $\theta_{\text{final}} \gets \mathcal{F}(\theta, \phi^{(T)})$ (Eq.~\ref{eq:prob_fuse})
  \STATE \textbf{Return} $\theta_{\text{final}}$
\end{algorithmic}
\end{algorithm}

\section{Experiments}
\label{sec:experiments}

\subsection{Experimental Setup}
\label{sec:exp_setup}
We evaluate our framework using \textbf{LLaMA2-7B}~\cite{touvron2023llama2} and \textbf{Mistral-7B-Instruct-v0.2}~\cite{jiang2023mistral7b} as the foundational LLMs. The experiments are designed around two key scenarios to simulate different real-world conditions.

\textbf{Datasets.} Our experiments comprise two phases with distinct data sources. In the initial LLM compression stage, we compress the base LLM into a proxy SLM using a publicly available instruction-following dataset \textbf{Alpaca}~\cite{alpaca}. In the subsequent federated fine-tuning stage, we train and evaluate across eight datasets from two widely adopted benchmarks. For question answering (QA), we use \textbf{OBQA}~\cite{mihaylov2018can}, \textbf{ARC-Challenge (ARC-C)}~\cite{clark2018think}, \textbf{ARC-Easy (ARC-E)}~\cite{clark2018think}, and \textbf{CommonsenseQA (CQA)}~\cite{talmor2019commonsenseqa}. For general language understanding, we use four tasks from the GLUE benchmark~\cite{wang2018glue}: \textbf{SST2}~\cite{socher-etal-2013-parsing}, \textbf{MRPC}~\cite{dolan2005automatically}, \textbf{RTE}~\cite{giampiccolo2007third}, and \textbf{MNLI}~\cite{williams2018broad}.

\textbf{Task Scenarios.} We consider two scenarios: (1) \textbf{Homogeneous Tasks} with a 4-client setup where each client is assigned a different partition of the same dataset (IID distribution), and (2) \textbf{Heterogeneous Tasks} with an 8-client setup where each client is assigned a unique and different dataset. All tasks are evaluated based on accuracy using the \textit{lm-evaluation-harness} package~\cite{eval-harness}. Detailed configurations are provided in Table \ref{tab:task_configuration}.

\textbf{Baselines.} We compare \textbf{FedProxy} against five baselines: (1) \textbf{ZeroShot}, the LLM's performance without any fine-tuning, serving as a lower bound; (2) \textbf{CentSFT}, the LLM fine-tuned on the data centrally, representing the performance upper bound; (3) \textbf{OT}~\cite{xiao2023offsite}, standard offsite-tuning with the first two and the last two decoders as the adapter; (4) \textbf{FedOT}~\cite{kuang2024federatedscope}, federated offsite-tuning with the first two and the last two decoders as the adapter; and (5) \textbf{FedBiOT}~\cite{wu2024fedbiot}, federated offsite-tuning with bi-level optimization, using the last four decoders as the adapter.

\begin{table}[ht]
  \centering
   \setlength{\tabcolsep}{3pt}
   \footnotesize
  
  \begin{tabular}{lccc}
     \toprule
    \textbf{Scenario}& \textbf{Architecture} & \textbf{Clients} & \textbf{Datasets} \\
   \midrule
    \multirow{4}{*}{Homogeneous} & \multirow{4}{*}{\makecell{1 Server: \\LLM \\4 Clients: \\Proxy SLM}}& Client 0 & Part of  Single Data\\
    
    & & Client 1 & Part of Single Data\\

    & & Client 2 & Part of Single Data\\

    & & Client 3 & Part of Single Data\\
    
    \midrule
    \multirow{8}{*}{Heterogeneous} & \multirow{8}{*}{\makecell{1 Server: \\LLM \\8 Clients: \\Proxy SLM}}& Client 0 & Complete OBQA  \\

    & & Client 1 & Complete ARC-E  \\

    & & Client 2 & Complete ARC-C  \\
 
    & & Client 3 & Complete CQA  \\
   
    & & Client 4 & Complete SST2  \\

    & & Client 5 & Complete MRPC \\

    & & Client 6 & Complete RTE  \\
  
    & & Client 7 & Complete MNLI \\
    \bottomrule
  \end{tabular}
  
  \caption{Configuration of Homogeneous and Heterogeneous Tasks in FedProxy.}
  \label{tab:task_configuration}
\end{table}

\subsection{Main Results}

We evaluate FedProxy by creating proxy SLMs with compression ratios of \textbf{30\% and 50\%}. The results, detailed in Tables \ref{tab:homo_compare_method} and \ref{tab:hetero_compare_method}, demonstrate our framework's consistent superiority across both models and task settings.

\textbf{Homogeneous Tasks}.
In the homogeneous setting (Table \ref{tab:homo_compare_method}), FedProxy establishes a commanding performance advantage. With LLaMA2-7B at 50\% compression, for instance, it achieves a QA accuracy of \textbf{0.5935} and a GLUE accuracy of \textbf{0.8038}. This marks a significant leap over OT-based baselines, with absolute improvements of \textbf{11.7} and \textbf{22.1} percentage points over FedOT on QA and GLUE tasks, respectively.  Crucially, FedProxy closes the performance gap to the centralized fine-tuning (CentSFT) upper bound, reaching \textbf{90.1\%} of its performance on QA and \textbf{90.7\%} on GLUE. This strong performance is consistently reflected across the Mistral-7B model, underscoring the framework's robustness and effectiveness.

\textbf{Heterogeneous Tasks}.
FedProxy's superiority is even more striking in the challenging heterogeneous setting (Table \ref{tab:hetero_compare_method}), where mitigating task interference is paramount. For LLaMA2-7B at 50\% compression, FedProxy achieves a QA accuracy of \textbf{0.6011} and a GLUE accuracy of \textbf{0.7944}. This represents a massive leap over FedOT, with absolute improvements of \textbf{12.6} and \textbf{20.1} percentage points on QA and GLUE tasks, respectively. Crucially, FedProxy again closes the performance gap to the centralized fine-tuning (CentSFT) upper bound, reaching \textbf{91.3\%} of its performance on QA and \textbf{89.6\%} on GLUE. This pattern of significant outperformance is consistently observed with the Mistral-7B model, powerfully validating the efficacy of our multi-stage aggregation strategy in mitigating parameter interference and proving its robustness in diverse federated ecosystems.

\begin{table*}[ht]
    \centering
    \setlength{\tabcolsep}{4pt}
    \footnotesize
    \begin{tabular}{llccccccccccc}
        \toprule
        & & & \multicolumn{5}{c}{\textbf{QA}} & \multicolumn{5}{c}{\textbf{GLUE}} \\
        \cmidrule(lr){4-8} \cmidrule(lr){9-13}
        \textbf{Model} & \textbf{Method} & \textbf{Ratio} & \textbf{OBQA} & \textbf{ARC-E} & \textbf{ARC-C} & \textbf{CQA} & \textit{\textbf{Avg}} & \textbf{SST2} & \textbf{MRPC} & \textbf{RTE} & \textbf{MNLI} & \textit{\textbf{Avg}} \\ 			
        \midrule
        
        \multirow{8}{*}{LLaMA2-7B}& ZeroShot & - & 0.322& 0.7656& 0.4445& 0.3079& 0.46& 0.4954& 0.6887& 0.6137& 0.3996& 0.5494\\
      
        & CentSFT & - & 0.522& 0.8102& 0.5188& 0.783& 0.6585& 0.9633& 0.8505& 0.8628& 0.8677& 0.8861\\
     
        & OT & 30 & 0.362& 0.7761& 0.4667& 0.3219& 0.4817& 0.4943& 0.701& 0.5704& 0.4938& 0.5649\\
    
        & FedOT & 30 & 0.362& 0.7782& 0.4642& 0.3268& 0.4828& 0.6468& 0.6936& 0.5523& 0.4609& 0.5884\\

         & FedBiOT & 30 & 0.336& 0.7626& 0.4454& 0.3129& 0.4642& 0.9002& 0.6912& 0.6101& 0.4737& 0.6688\\
      
        & \textbf{FedProxy} & \textbf{30} & \textbf{0.446}& \textbf{0.7963}& \textbf{0.5068}& \textbf{0.7551}& \textbf{0.6261}&\textbf{ 0.961}& \textbf{0.7696}& \textbf{0.8809}& \textbf{0.8596}& \textbf{0.8678}\\
        
        & OT & 50 & 0.35& 0.7753& 0.4437& 0.3112& 0.4701& 0.539& 0.6936& 0.6245& 0.36& 0.5543\\
    
        & FedOT & 50 & 0.354& 0.7715& 0.4462& 0.335& 0.4767& 0.6206& 0.7083& 0.6101& 0.3941& 0.5833\\

         & FedBiOT & 50 & 0.334& 0.7652& 0.4488& 0.3088& 0.4642& 0.8257& 0.6936& 0.6029& 0.4509& 0.6433\\
      
        & \textbf{FedProxy} & \textbf{50} & \textbf{0.39}& \textbf{0.7866}& \textbf{0.471}& \textbf{0.7265}& \textbf{0.5935}& \textbf{0.9656}& \textbf{0.7696}& \textbf{0.8592}& \textbf{0.6207}& \textbf{0.8038}\\

        \midrule

       \multirow{8}{*}{Mistral-7B} & ZeroShot & - & 0.368& 0.8114& 0.5468& 0.6806& 0.6017& 0.867& 0.7304& 0.722& 0.5952& 0.7287\\
      
        & CentSFT & - & 0.536& 0.838& 0.5734& 0.8288& 0.6941& 0.9656& 0.8922& 0.8881& 0.8835& 0.9074\\
     
        & OT & 30 & 0.44& 0.8009& 0.5162& 0.7445& 0.6254& 0.9404& 0.7353& 0.7834& 0.7264& 0.7964\\
    
        & FedOT & 30 & 0.428& 0.8266& 0.5666& 0.7363& 0.6394& 0.9415& 0.75& 0.7581& 0.7562& 0.8015\\

        & FedBiOT & 30 & 0.392& 0.8258& 0.5495& 0.697& 0.6161& 0.9335& 0.75& 0.7292& 0.6714& 0.771\\
      
        & \textbf{FedProxy} & \textbf{30} & \textbf{0.494}& \textbf{0.8396}& \textbf{0.5657}& \textbf{0.7985}& \textbf{0.6745}& \textbf{0.9667}& \textbf{0.8456}& \textbf{0.87}& \textbf{0.8848}& \textbf{0.8918}\\
        
        & OT & 50 & 0.422& 0.7955& 0.5452& 0.688& 0.6127& 0.9048& 0.7353& 0.7509& 0.6174& 0.7521\\
    
        & FedOT & 50 & 0.4& 0.8182& 0.5683& 0.6888& 0.6188& 0.8819& 0.7426& 0.7329& 0.5934& 0.7377\\

         & FedBiOT & 50 & 0.37& 0.8228& 0.5469& 0.6945& 0.6086& 0.9025& 0.7304& 0.722& 0.6082& 0.7408\\
      
        & \textbf{FedProxy} & \textbf{50} & \textbf{0.462}& \textbf{0.8573}& \textbf{0.5998}& \textbf{0.7576}& \textbf{0.6692}& \textbf{0.9644}& \textbf{0.7672}& \textbf{0.8123}& \textbf{0.834}& \textbf{0.8445}\\

        \bottomrule
    \end{tabular}
    \caption{Method Performance Comparison in the Homogeneous Tasks Setting.}
    \label{tab:homo_compare_method}
\end{table*}

\begin{table*}[ht]
    \centering
    \setlength{\tabcolsep}{4pt}
    \footnotesize
    \begin{tabular}{llccccccccccc}
        \toprule
        & & & \multicolumn{5}{c}{\textbf{QA}} & \multicolumn{5}{c}{\textbf{GLUE}} \\
        \cmidrule(lr){4-8} \cmidrule(lr){9-13}
        \textbf{Model} & \textbf{Method} & \textbf{Ratio} & \textbf{OBQA} & \textbf{ARC-E} & \textbf{ARC-C} & \textbf{CQA} & \textit{\textbf{Avg}} & \textbf{SST2} & \textbf{MRPC} & \textbf{RTE} & \textbf{MNLI} & \textit{\textbf{Avg}} \\ 			
        \midrule
        
        \multirow{8}{*}{LLaMA2-7B}& ZeroShot & - & 0.322& 0.7656& 0.4445& 0.3079& 0.46& 0.4954& 0.6887& 0.6137& 0.3996& 0.5494\\
      
        & CentSFT & - & 0.522& 0.8102& 0.5188& 0.783& 0.6585& 0.9633& 0.8505& 0.8628& 0.8677& 0.8861\\
     
        & OT & 30 & 0.362& 0.7761& 0.4667& 0.3219& 0.4817& 0.4943& 0.701& 0.5704& 0.4938& 0.5649\\
    
        & FedOT & 30 & 0.344& 0.7803& 0.4659& 0.362& 0.4881& 0.828& 0.6887& 0.6931& 0.4107& 0.6551\\

        & FedBiOT & 30 & 0.334& 0.7391& 0.4249& 0.2793& 0.4443& 0.8417& 0.6324& 0.6209& 0.4652& 0.6401\\
      
        & \textbf{FedProxy} & \textbf{30} & \textbf{0.424}& \textbf{0.7837}& \textbf{0.5162}& \textbf{0.7658}& \textbf{0.6224}&\textbf{0.9541}& \textbf{0.8113}& \textbf{0.8267}& \textbf{0.8056}& \textbf{0.8494}\\
        
        & OT & 50 & 0.35& 0.7753& 0.4437& 0.3112& 0.4701& 0.539& 0.6936& 0.6245& 0.36& 0.5543\\
    
        & FedOT & 50 & 0.35& 0.766& 0.442& 0.3419& 0.475& 0.6514& 0.6863& 0.6245& 0.4108& 0.5933\\

         & FedBiOT & 50 & 0.332& 0.7523& 0.4369& 0.3178& 0.4598& 0.5126& 0.6936& 0.6137& 0.4652& 0.5713\\
      
        & \textbf{FedProxy} & \textbf{50} & \textbf{0.382}& \textbf{0.7866}& \textbf{0.4872}& \textbf{0.7486}& \textbf{0.6011}& \textbf{0.9541}& \textbf{0.7721}& \textbf{0.8484}& \textbf{0.6029}& \textbf{0.7944}\\

        \midrule

       \multirow{8}{*}{Mistral-7B} & ZeroShot & - & 0.368& 0.8114& 0.5468& 0.6806& 0.6017& 0.867& 0.7304& 0.722& 0.5952& 0.7287\\
      
        & CentSFT & - & 0.536& 0.838& 0.5734& 0.8288& 0.6941& 0.9656& 0.8922& 0.8881& 0.8835& 0.9074\\
     
        & OT & 30 & 0.44& 0.8009& 0.5162& 0.7445& 0.6254& 0.9404& 0.7353& 0.7834& 0.7264& 0.7964\\
    
        & FedOT & 30 & 0.404& 0.8405& 0.5794& 0.7113& 0.6338& 0.93& 0.7377& 0.7834& 0.7066& 0.7894\\

        & FedBiOT & 30 & 0.38& 0.8085& 0.5452& 0.6945& 0.6071& 0.9128& 0.7304& 0.7401& 0.6258& 0.7523\\
      
        & \textbf{FedProxy} & \textbf{30} & \textbf{0.448}& \textbf{0.8211}& \textbf{0.5794}& \textbf{0.8092}& \textbf{0.6644}& \textbf{0.9564}& \textbf{0.8382}& \textbf{0.8592}& \textbf{0.8408}& \textbf{0.8737}\\
        
        & OT & 50 & 0.422& 0.7955& 0.5452& 0.688& 0.6127& 0.9048& 0.7353& 0.7509& 0.6174& 0.7521\\
    
        & FedOT & 50 & 0.386& 0.827& 0.5683& 0.7027& 0.621& 0.9002& 0.7181& 0.7148& 0.6101& 0.7358\\

        & FedBiOT & 50 & 0.358& 0.8093& 0.5538& 0.6921& 0.6033& 0.914& 0.7304& 0.7184& 0.6036& 0.7416\\
      
        & \textbf{FedProxy} & \textbf{50} & \textbf{0.434}& \textbf{0.8472}& \textbf{0.6109}& \textbf{0.7625}& \textbf{0.6637}& \textbf{0.961}& \textbf{0.8039}& \textbf{0.8592}& \textbf{0.8091}& \textbf{0.8583}\\

        \bottomrule
    \end{tabular}
    \caption{Method Performance Comparison in the Heterogeneous Tasks Setting.}
    \label{tab:hetero_compare_method}
\end{table*}

\textbf{Cost-Performance Trade-off Analysis}.
\label{sec:appendix-cost}
FedProxy is designed for a superior cost-performance trade-off. Compared to methods like FedOT, it strategically increases client-side computation and iterative communication to achieve significant performance gains, while keeping initial communication costs comparable. As detailed in Table~\ref{tab:cost_analysis}, this approach yields performance substantially closer to the centralized upper bound, justifying the higher resource investment.

The costs can be broken down as follows:
\begin{itemize}
    \item \textbf{Initial Communication Cost:} Both methods require an initial model transfer. For FedOT, this is a small emulator and adapter; for FedProxy, it is the compressed proxy SLM. The size of this one-time transfer is designed to be comparable.
    \item \textbf{Iterative Communication Cost:} In each round, clients communicate the trained LoRA parameters. Since FedProxy fine-tunes a larger proxy SLM, its LoRA updates are larger than FedOT's, leading to higher iterative costs.
    \item \textbf{Client Computation Cost:} FedProxy clients fine-tune the larger proxy SLM, demanding more computational power than FedOT clients, which in turn delivers superior performance.
\end{itemize}

To provide a clear, quantitative comparison, Table~\ref{tab:cost_analysis} details the costs for the LLaMA2-7B model with 50\% compression. The table quantifies the costs in terms of parameter counts. The \textbf{Initial Communication Cost} for both FedProxy and FedOT involves transferring a $\sim$3.5 billion parameter model. The \textbf{Iterative Communication Cost} (LoRA parameters) is approximately 76.24M for FedProxy, about four times higher than FedOT's $\sim$19.06M. Similarly, the \textbf{Client Computation Cost} (trainable parameters) is $\sim$38.12M for FedProxy, also four times that of FedOT's $\sim$9.53M. These figures underscore that FedProxy's performance gain is achieved by investing more in iterative communication and client-side computation.

\begin{table*}[ht]
    \centering
    \footnotesize
    \setlength{\tabcolsep}{4pt}
    
    \begin{tabular}{lcccc}
        \toprule
        \textbf{Method} & \textbf{Initial Comm Cost} & \textbf{Iterative Comm Cost} & \textbf{Client Compute Cost}  & \textbf{Performance (ALL Avg.)}\\
        \midrule
        CentSFT & $\sim$7B & N/A & $\sim$76.24M & 0.7723\\
        FedOT & $\sim$3.5B & $\sim$19.06M & $\sim$9.53M & 0.5300\\
        \textbf{FedProxy} & \textbf{$\sim$3.5B} & \textbf{$\sim$76.24M} & \textbf{$\sim$38.12M} & 0.6987\\
        \bottomrule
    \end{tabular}
    \caption{Cost-performance analysis in terms of parameter counts. FedProxy's performance gain is achieved by investing more in iterative communication and client-side computation, while maintaining a comparable initial model download size.}
    \label{tab:cost_analysis}
\end{table*}

\subsection{Ablation Study}

\subsubsection{Impact of Aggregation Components in the Heterogeneous Setting}

To analyze our multi-stage aggregation strategy, we perform an ablation study in a heterogeneous setting at 50\% compression. We compare six configurations (Table \ref{tab:ablation_aggregation}): (1) \textbf{FedProxy (Full)}, our complete method with PCR and H-TIES; (2) \textbf{w/o PCR}, H-TIES server merge only; (3) \textbf{w/o H-TIES}, client PCR with standard FedAvg; (4) \textbf{FedAvg}~\cite{mcmahan2017communication}, standard federated averaging baseline; (5) \textbf{FedProx}~\cite{li2020federated}, FedAvg variant baseline with a proximal term; and (6) \textbf{TIES-Merging}~\cite{yadav2023ties}, server-only TIES baseline.

The results clearly demonstrate the synergistic value of our design. The \textbf{FedAvg}, \textbf{FedProx} and \textbf{TIES-Merging} show substantially degraded performance, confirming that standard aggregation methods are insufficient to handle parameter interference. Removing either of our key components results in a significant performance drop: \textbf{w/o PCR} struggles to align client updates effectively, while \textbf{w/o H-TIES} fails to optimally merge models even with regularized clients. The \textbf{full FedProxy model} consistently and significantly outperforms all ablated and baseline configurations, proving that both conflict-aware client regularization and heterogeneity-aware server merging are critical, complementary components for achieving robust performance.

\begin{table}[ht]
    \centering
    \footnotesize
    \setlength{\tabcolsep}{4pt}
    \begin{tabular}{llccc}
        \toprule
        \textbf{Model}  & \textbf{Method}& {\textbf{QA}} & {\textbf{GLUE}}  & {\textbf{ALL}}\\
        \midrule
        \multirow{6}{*}{LLaMA2-7B}& \textbf{FedProxy (Full)}  & 0.6011& 0.7944 &\textbf{0.6977}\\
        
        & w/o PCR& 0.5948&  0.7860 &0.6904\\

        & w/o H-TIES& 0.5962& 0.7615 &0.6788\\
        \cmidrule(lr){2-5}
        & FedAvg& \textbf{0.6065}& 0.7551 &0.6808\\
        & FedProx& 0.5959& 0.7601& 0.6780\\
        & TIES-Merging& 0.5911& \textbf{0.8008}&0.6959\\

        \midrule
        \multirow{6}{*}{Mistral-7B}& \textbf{FedProxy (Full)} & \textbf{0.6637}& \textbf{0.8583}&\textbf{0.7610}\\
        
        & w/o PCR& 0.6626&  0.8477&0.7551\\

        & w/o H-TIES& 0.6595& 0.8369&0.7482\\
        \cmidrule(lr){2-5}
        & FedAvg& 0.6617& 0.8261&0.7439\\
        & FedProx& 0.6546& 0.8357& 0.7451\\
        & TIES-Merging& 0.6543& 0.8547&0.7545\\
    
        \bottomrule
    \end{tabular}
    \caption{Ablation study of the aggregation components in the Heterogeneous Tasks Setting. 
    We report average accuracy on QA, GLUE, and ALL.
    Our full method achieves the best overall performance.
    }
    \label{tab:ablation_aggregation}
\end{table}

\subsubsection{Architectural Advantage of the Proxy SLM}
To isolate the architectural benefits of our proxy SLM, we conduct an ablation study against the Offsite-Tuning (OT) paradigm. We introduce a stronger baseline, \textbf{FedOT-ET (Emulator Trained)}, where clients train both the emulator and adapters, in contrast to standard OT where only adapters are trained. This experiment, conducted in the homogeneous setting with a 50\% compression ratio, aims to answer: \textbf{\textit{Is FedProxy's advantage merely due to training more parameters, or does its architectural design offer a fundamental benefit?}}  

As shown in Table \ref{tab:ablation_proxyslm}, while FedOT-ET slightly improves on FedOT, FedProxy significantly outperforms both. This confirms that the superiority of FedProxy originates from its cohesive architectural design: fine-tuning an integrated proxy SLM is inherently more effective than training a disjointed adapter-emulator structure, which thereby underscores the intrinsic limitations of the OT paradigm.

\begin{table}[ht]
    \centering
    \footnotesize
    \setlength{\tabcolsep}{4pt}
    \begin{tabular}{llccc}
        \toprule
        \textbf{Model}  & \textbf{Method}& {\textbf{QA}} & {\textbf{GLUE}}  & {\textbf{ALL}}\\
        \midrule
        \multirow{3}{*}{LLaMA2-7B}& \textbf{FedProxy}  & \textbf{0.5935}& \textbf{0.8038}&\textbf{0.6987}\\
        
        & FedOT-ET& 0.4877&  0.5751&0.5314\\

        & FedOT& 0.4767& 0.5833&0.5300\\

        \midrule
        \multirow{3}{*}{Mistral-7B}& \textbf{FedProxy} & \textbf{0.6692}& \textbf{0.8445}&\textbf{0.7568}\\
        
        & FedOT-ET& 0.6251&  0.7455&0.6853\\

        & FedOT& 0.6188& 0.7377&0.6783\\
        
        \bottomrule
    \end{tabular}
    \caption{Architectural comparison with Offsite-Tuning (OT) variants. 
    FedOT-ET (Emulator Trained) is a stronger baseline where the emulator is also trained.
    The results demonstrate that FedProxy's cohesive proxy SLM architecture significantly outperforms the adapter-based OT structure, even when the OT emulator is unfrozen.
    }
    \label{tab:ablation_proxyslm}
\end{table}

\subsubsection{IP Protection Analysis}
In FedProxy, \texttt{IP protection} refers to reducing direct exposure of the proprietary backbone (e.g., by deploying a compressed proxy on clients), rather than to formal cryptographic or information-theoretic guarantees. To assess this pragmatic objective, we analyze the standalone performance of proxy SLMs.
We conducted an ablation study under a homogeneous setting with a 50\% compression ratio. The results are presented in Table~\ref{tab:slm_standalone_performance}.
The results in Table~\ref{tab:slm_standalone_performance} lead to two key conclusions. First, the initial zero-shot proxy SLM (\texttt{FedProxy-SLM-ZS}) performs significantly worse than the original zero-shot LLM (\texttt{LLM-ZS}) across all benchmarks (e.g., for LLaMA2-7B, an overall score of 0.3727 vs. 0.5047). This confirms that the compressed SLM is not a powerful standalone model, mitigating IP risks. Second, after federated fine-tuning, the final \textbf{\texttt{FedProxy-LLM}} substantially outperforms the original \texttt{LLM-ZS} (e.g., for LLaMA2-7B, 0.6987 vs. 0.5047). This demonstrates the effectiveness of our framework: the proxy SLM serves as an effective vehicle for federated learning, and the knowledge fusion mechanism successfully integrates the gains back into the full model.

\begin{table}[ht]
    \centering
    \footnotesize
    \setlength{\tabcolsep}{3pt}
    \begin{tabular}{llccc}
        \toprule
        \textbf{Model}  & \textbf{Method}& {\textbf{QA}} & {\textbf{GLUE}}  & {\textbf{ALL}}\\
        \midrule
        \multirow{4}{*}{LLaMA2-7B} & LLM-ZS & 0.4600 & 0.5494& 0.5047\\
        & \textbf{\texttt{FedProxy-SLM-ZS}} & \textbf{0.2242} & \textbf{0.5213}& \textbf{0.3727}\\
        & \texttt{FedProxy-SLM-SFT} & 0.4260 & 0.8603& 0.6432\\
        & \texttt{FedProxy-LLM} & 0.5935 & 0.8038& 0.6987\\

        \midrule
        \multirow{4}{*}{Mistral-7B} & LLM-ZS & 0.6017& 0.7287& 0.6652\\
        & \textbf{\texttt{FedProxy-SLM-ZS}} &0.2510& 0.4419& 0.3464\\
        & \texttt{FedProxy-SLM-SFT} & 0.4713& 0.8564& 0.6638\\
        & \texttt{FedProxy-LLM} & 0.6692& 0.8445& 0.7568\\
        
        \bottomrule
    \end{tabular}
    \caption{Analysis of standalone SLM performance. The initial compressed SLM (\texttt{FedProxy-SLM-ZS}) is significantly weaker than the original LLM.}
    \label{tab:slm_standalone_performance}
\end{table}

\subsubsection{Impact of Compression Ratio}
We evaluated various compression ratios and found that 50\% strikes an optimal balance between model performance and resource efficiency. While lower ratios offer marginal gains at a high resource cost, higher ratios lead to a significant performance drop. For detailed comparison results and analysis, please refer to Appendix~\ref{sec:appendix-compression} and Table ~\ref{tab:ablation_compression}.

\subsubsection{Comparison with ProxyTuning}
To evaluate the efficacy of our parameter-space "plug-in" fusion mechanism, we compare it against ProxyTuning's decoding-time fusion strategy~\cite{liu2024tuning, gao2024fedpt}. Our analysis reveals a task-dependent performance trade-off: FedProxy demonstrates superior performance on generative reasoning (QA) tasks, while ProxyTuning exhibits advantages on discriminative benchmarks (GLUE). Importantly, FedProxy achieves competitive performance with \textbf{zero additional inference overhead}, and performance degradation at high compression ratios is primarily due to compression limitations rather than fusion mechanisms. For detailed comparison results and analysis, please refer to Appendix~\ref{sec:appendix-proxytuning-comparison} and Table~\ref{tab:proxytuning_comparison}.

\section{Conclusions}
In this work, we introduced \textbf{FedProxy}, a federated LLM fine-tuning framework that addresses the core challenges of LLM IP protection, client-side constraints, data privacy, and data heterogeneity through a systematic, three-stage solution encompassing compression, aggregation, and fusion. Our framework replaces weak adapters with a powerful, compressed proxy SLM, mitigates parameter interference through a multi-stage aggregation strategy, and effectively fuses learned knowledge back into the global LLM. Extensive experiments show that FedProxy significantly outperforms OT-based methods and approaches the performance of centralized fine-tuning, establishing a new standard for secure, efficient, and high-performance federated LLM adaptation.

\section*{Limitations}
Despite its effectiveness, FedProxy presents several avenues for further refinement. 
First, the compression-performance trade-off remains a constraint: while a 50\% compression ratio is effective, higher ratios (e.g., 70\%) lead to performance degradation due to reduced parameter capacity. Although unstructured pruning~\cite{wang2020picking,frantar2023sparsegpt,sun2024a} could offer higher sparsity, it introduces challenges in hardware efficiency and poses a non-trivial task of fusing sparse updates back into a dense backbone. 
Second, the computational scalability of our conflict-aware aggregation currently scales quadratically with client density ($O(K^2)$). For massive-scale deployments, exploring Approximate Nearest Neighbor (ANN) vector search ~\cite{arya1998optimal}, client filtering, or hierarchical aggregation methods ~\cite{liu2020client,wang2021resource} will be essential to mitigate this overhead. 
Third, the quality of the proxy SLM is intrinsically linked to the public dataset used for distillation. While the framework is flexible and compatible with various privacy-preserving compression techniques~\cite{fan2025ppc}, servers may need to leverage domain-specific public datasets or synthetic data generation to enhance target-domain generalization and ensure the universality of the proxy model. 
Finally, the long-term stability of repeated knowledge fusion requires further longitudinal study to ensure the global backbone maintains its original capabilities without suffering from catastrophic forgetting ~\cite{french1999catastrophic,kirkpatrick2017overcoming} or representation drift~\cite{patil2026detecting}.

\bibliography{ref}

@inproceedings{mcmahan2017communication,
  title={Communication-efficient learning of deep networks from decentralized data},
  author={McMahan, Brendan and Moore, Eider and Ramage, Daniel and Hampson, Seth and y Arcas, Blaise Aguera},
  booktitle={Artificial intelligence and statistics},
  pages={1273--1282},
  year={2017},
  organization={PMLR}
}

@article{yang2019federated,
  title={Federated learning},
  author={Yang, Qiang and Liu, Yang and Cheng, Yong and Kang, Yan and Chen, Tianjian and Yu, Han},
  journal={Synthesis Lectures on Artificial Intelligence and Machine Learning},
  volume={13},
  number={3},
  pages={1--207},
  year={2019},
  publisher={Morgan \& Claypool Publishers}
}

@article{touvron2023llama2,
  title={Llama 2: Open foundation and fine-tuned chat models},
  author={Touvron, Hugo and Martin, Louis and Stone, Kevin and Albert, Peter and Almahairi, Amjad and Babaei, Yasmine and Bashlykov, Nikolay and Batra, Soumya and Bhargava, Prajjwal and Bhosale, Shruti and others},
  journal={arXiv preprint arXiv:2307.09288},
  year={2023}
}

@article{xiao2023offsite,
  title={Offsite-tuning: Transfer learning without full model},
  author={Xiao, Guangxuan and Lin, Ji and Han, Song},
  journal={arXiv preprint arXiv:2302.04870},
  year={2023}
}

@misc{alpaca,
  author = {Rohan Taori and Ishaan Gulrajani and Tianyi Zhang and Yann Dubois and Xuechen Li and Carlos Guestrin and Percy Liang and Tatsunori B. Hashimoto },
  title = {Stanford Alpaca: An Instruction-following LLaMA model},
  year = {2023},
  publisher = {GitHub},
  journal = {GitHub repository},
  howpublished = {\url{https://github.com/tatsu-lab/stanford_alpaca}},
}

@inproceedings{mihaylov2018can,
  title={Can a Suit of Armor Conduct Electricity? A New Dataset for Open Book Question Answering},
  author={Mihaylov, Todor and Clark, Peter and Khot, Tushar and Sabharwal, Ashish},
  booktitle={Proceedings of the 2018 Conference on Empirical Methods in Natural Language Processing},
  pages={2381--2391},
  year={2018}
}

@article{clark2018think,
  title={Think you have solved question answering? try arc, the ai2 reasoning challenge},
  author={Clark, Peter and Cowhey, Isaac and Etzioni, Oren and Khot, Tushar and Sabharwal, Ashish and Schoenick, Carissa and Tafjord, Oyvind},
  journal={arXiv preprint arXiv:1803.05457},
  year={2018}
}

@article{fan2023fate-llm,
  title={Fate-llm: A industrial grade federated learning framework for large language models},
  author={Fan, Tao and Kang, Yan and Ma, Guoqiang and Chen, Weijing and Wei, Wenbin and Fan, Lixin and Yang, Qiang},
  journal={arXiv preprint arXiv:2310.10049},
  year={2023}
}

@inproceedings{talmor2019commonsenseqa,
  title={CommonsenseQA: A Question Answering Challenge Targeting Commonsense Knowledge},
  author={Talmor, Alon and Herzig, Jonathan and Lourie, Nicholas and Berant, Jonathan},
  booktitle={Proceedings of the 2019 Conference of the North American Chapter of the Association for Computational Linguistics: Human Language Technologies, Volume 1 (Long and Short Papers)},
  pages={4149--4158},
  year={2019}
}

@inproceedings{lhoest-etal-2021-datasets,
    title = "Datasets: A Community Library for Natural Language Processing",
    author = "Lhoest, Quentin  and
      Villanova del Moral, Albert  and
      Jernite, Yacine  and
      Thakur, Abhishek  and
      von Platen, Patrick  and
      Patil, Suraj  and
      Chaumond, Julien  and
      Drame, Mariama  and
      Plu, Julien  and
      Tunstall, Lewis  and
      Davison, Joe  and
      {\v{S}}a{\v{s}}ko, Mario  and
      Chhablani, Gunjan  and
      Malik, Bhavitvya  and
      Brandeis, Simon  and
      Le Scao, Teven  and
      Sanh, Victor  and
      Xu, Canwen  and
      Patry, Nicolas  and
      McMillan-Major, Angelina  and
      Schmid, Philipp  and
      Gugger, Sylvain  and
      Delangue, Cl{\'e}ment  and
      Matussi{\`e}re, Th{\'e}o  and
      Debut, Lysandre  and
      Bekman, Stas  and
      Cistac, Pierric  and
      Goehringer, Thibault  and
      Mustar, Victor  and
      Lagunas, Fran{\c{c}}ois  and
      Rush, Alexander  and
      Wolf, Thomas",
    booktitle = "Proceedings of the 2021 Conference on Empirical Methods in Natural Language Processing: System Demonstrations",
    month = nov,
    year = "2021",
    address = "Online and Punta Cana, Dominican Republic",
    publisher = "Association for Computational Linguistics",
    url = "https://aclanthology.org/2021.emnlp-demo.21",
    pages = "175--184",
    eprint={2109.02846},
    archivePrefix={arXiv},
    primaryClass={cs.CL},
}

@misc{eval-harness,
  author       = {Gao, Leo and Tow, Jonathan and Abbasi, Baber and Biderman, Stella and Black, Sid and DiPofi, Anthony and Foster, Charles and Golding, Laurence and Hsu, Jeffrey and Le Noac'h, Alain and Li, Haonan and McDonell, Kyle and Muennighoff, Niklas and Ociepa, Chris and Phang, Jason and Reynolds, Laria and Schoelkopf, Hailey and Skowron, Aviya and Sutawika, Lintang and Tang, Eric and Thite, Anish and Wang, Ben and Wang, Kevin and Zou, Andy},
  title        = {A framework for few-shot language model evaluation},
  month        = 12,
  year         = 2023,
  publisher    = {Zenodo},
  version      = {v0.4.0},
  doi          = {10.5281/zenodo.10256836},
  url          = {https://zenodo.org/records/10256836}
}

@inproceedings{fan2025fedmkt,
  title={FedMKT: Federated Mutual Knowledge Transfer for Large and Small Language Models},
  author={Fan, Tao and Ma, Guoqiang and Kang, Yan and Gu, Hanlin and Song, Yuanfeng and Fan, Lixin and Chen, Kai and Yang, Qiang},
  booktitle={Proceedings of the 31st International Conference on Computational Linguistics},
  pages={243--255},
  year={2025}
}

@article{OpenAI2023GPT4TR,
  title   = {GPT-4 Technical Report},
  author  = {OpenAI},
  journal = {ArXiv},
  year    = {2023},
  volume  = {abs/2303.08774}
}

@inproceedings{
wang2020picking,
title={Picking Winning Tickets Before Training by Preserving Gradient Flow},
author={Chaoqi Wang and Guodong Zhang and Roger Grosse},
booktitle={International Conference on Learning Representations},
year={2020},
url={https://openreview.net/forum?id=SkgsACVKPH}
}

@inproceedings{
sun2024a,
title={A Simple and Effective Pruning Approach for Large Language Models},
author={Mingjie Sun and Zhuang Liu and Anna Bair and J Zico Kolter},
booktitle={The Twelfth International Conference on Learning Representations},
year={2024},
url={https://openreview.net/forum?id=PxoFut3dWW}
}

@inproceedings{frantar2023sparsegpt,
  title={Sparsegpt: Massive language models can be accurately pruned in one-shot},
  author={Frantar, Elias and Alistarh, Dan},
  booktitle={International Conference on Machine Learning},
  pages={10323--10337},
  year={2023},
  organization={PMLR}
}

@article{men2024shortgpt,
  title={Shortgpt: Layers in large language models are more redundant than you expect},
  author={Men, Xin and Xu, Mingyu and Zhang, Qingyu and Wang, Bingning and Lin, Hongyu and Lu, Yaojie and Han, Xianpei and Chen, Weipeng},
  journal={arXiv preprint arXiv:2403.03853},
  year={2024}
}

@inproceedings{kuang2024federatedscope,
  title={Federatedscope-llm: A comprehensive package for fine-tuning large language models in federated learning},
  author={Kuang, Weirui and Qian, Bingchen and Li, Zitao and Chen, Daoyuan and Gao, Dawei and Pan, Xuchen and Xie, Yuexiang and Li, Yaliang and Ding, Bolin and Zhou, Jingren},
  booktitle={Proceedings of the 30th ACM SIGKDD Conference on Knowledge Discovery and Data Mining},
  pages={5260--5271},
  year={2024}
}

@inproceedings{wortsman2022model,
  title={Model soups: averaging weights of multiple fine-tuned models improves accuracy without increasing inference time},
  author={Wortsman, Mitchell and Ilharco, Gabriel and Gadre, Samir Ya and Roelofs, Rebecca and Gontijo-Lopes, Raphael and Morcos, Ari S and Namkoong, Hongseok and Farhadi, Ali and Carmon, Yair and Kornblith, Simon and others},
  booktitle={International conference on machine learning},
  pages={23965--23998},
  year={2022},
  organization={PMLR}
}

@article{matena2022merging,
  title={Merging models with fisher-weighted averaging},
  author={Matena, Michael S and Raffel, Colin A},
  journal={Advances in Neural Information Processing Systems},
  volume={35},
  pages={17703--17716},
  year={2022}
}

@inproceedings{
jin2023dataless,
title={Dataless Knowledge Fusion by Merging Weights of Language Models},
author={Xisen Jin and Xiang Ren and Daniel Preotiuc-Pietro and Pengxiang Cheng},
booktitle={The Eleventh International Conference on Learning Representations },
year={2023},
url={https://openreview.net/forum?id=FCnohuR6AnM}
}

@inproceedings{
ilharco2023editing,
title={Editing models with task arithmetic},
author={Gabriel Ilharco and Marco Tulio Ribeiro and Mitchell Wortsman and Ludwig Schmidt and Hannaneh Hajishirzi and Ali Farhadi},
booktitle={The Eleventh International Conference on Learning Representations },
year={2023},
url={https://openreview.net/forum?id=6t0Kwf8-jrj}
}

@article{yadav2023ties,
  title={Ties-merging: Resolving interference when merging models},
  author={Yadav, Prateek and Tam, Derek and Choshen, Leshem and Raffel, Colin A and Bansal, Mohit},
  journal={Advances in Neural Information Processing Systems},
  volume={36},
  pages={7093--7115},
  year={2023}
}

@inproceedings{yu2024language,
  title={Language models are super mario: Absorbing abilities from homologous models as a free lunch},
  author={Yu, Le and Yu, Bowen and Yu, Haiyang and Huang, Fei and Li, Yongbin},
  booktitle={Forty-first International Conference on Machine Learning},
  year={2024}
}

@inproceedings{
wang2018glue,
title={{GLUE}: A Multi-Task Benchmark and Analysis Platform for Natural Language Understanding},
author={Alex Wang and Amanpreet Singh and Julian Michael and Felix Hill and Omer Levy and Samuel R. Bowman},
booktitle={International Conference on Learning Representations},
year={2019},
url={https://openreview.net/forum?id=rJ4km2R5t7},
}

@inproceedings{dolan2005automatically,
  title={Automatically constructing a corpus of sentential paraphrases},
  author={Dolan, Bill and Brockett, Chris},
  booktitle={Third international workshop on paraphrasing (IWP2005)},
  year={2005}
}

@inproceedings{giampiccolo2007third,
  title={The third pascal recognizing textual entailment challenge},
  author={Giampiccolo, Danilo and Magnini, Bernardo and Dagan, Ido and Dolan, William B},
  booktitle={Proceedings of the ACL-PASCAL workshop on textual entailment and paraphrasing},
  pages={1--9},
  year={2007}
}

@inproceedings{williams2018broad,
  title={A Broad-Coverage Challenge Corpus for Sentence Understanding through Inference},
  author={Williams, Adina and Nangia, Nikita and Bowman, Samuel},
  booktitle={Proceedings of the 2018 Conference of the North American Chapter of the Association for Computational Linguistics: Human Language Technologies, Volume 1 (Long Papers)},
  pages={1112--1122},
  year={2018}
}

@inproceedings{socher-etal-2013-parsing,
    title = "Parsing with Compositional Vector Grammars",
    author = "Socher, Richard  and
      Bauer, John  and
      Manning, Christopher D.  and
      Ng, Andrew Y.",
    editor = "Schuetze, Hinrich  and
      Fung, Pascale  and
      Poesio, Massimo",
    booktitle = "Proceedings of the 51st Annual Meeting of the Association for Computational Linguistics (Volume 1: Long Papers)",
    month = aug,
    year = "2013",
    address = "Sofia, Bulgaria",
    publisher = "Association for Computational Linguistics",
    url = "https://aclanthology.org/P13-1045/",
    pages = "455--465"
}

@misc{jiang2023mistral7b,
      title={Mistral 7B}, 
      author={Albert Q. Jiang and Alexandre Sablayrolles and Arthur Mensch and Chris Bamford and Devendra Singh Chaplot and Diego de las Casas and Florian Bressand and Gianna Lengyel and Guillaume Lample and Lucile Saulnier and Lélio Renard Lavaud and Marie-Anne Lachaux and Pierre Stock and Teven Le Scao and Thibaut Lavril and Thomas Wang and Timothée Lacroix and William El Sayed},
      year={2023},
      eprint={2310.06825},
      archivePrefix={arXiv},
      primaryClass={cs.CL},
      url={https://arxiv.org/abs/2310.06825}, 
}

@inproceedings{yao2025scaleot,
  title={ScaleOT: Privacy-utility-scalable Offsite-tuning with Dynamic LayerReplace and Selective Rank Compression},
  author={Yao, Kai and Tan, Zhaorui and Ye, Tiandi and Li, Lichun and Zhao, Yuan and Liu, Wenyan and Wang, Wei and Zhu, Jianke},
  booktitle={Proceedings of the AAAI Conference on Artificial Intelligence},
  volume={39},
  number={21},
  pages={22074--22082},
  year={2025}
}

@inproceedings{yao2025gradot,
  title={GradOT: Training-free Gradient-preserving Offsite-tuning for Large Language Models},
  author={Yao, Kai and Tan, Zhaorui and Gao, Penglei and Li, Lichun and Wu, Kaixin and Wang, Yinggui and Zhao, Yuan and Ji, Yixin and Zhu, Jianke and Wang, Wei},
  booktitle={Proceedings of the 63rd Annual Meeting of the Association for Computational Linguistics (Volume 1: Long Papers)},
  pages={5115--5130},
  year={2025}
}

@inproceedings{zhang2023crash,
  title={CRaSh: Clustering, Removing, and Sharing Enhance Fine-tuning without Full Large Language Model},
  author={Zhang, Kaiyan and Ding, Ning and Qi, Biqing and Zhu, Xuekai and Long, Xinwei and Zhou, Bowen},
  booktitle={Proceedings of the 2023 Conference on Empirical Methods in Natural Language Processing},
  pages={9612--9637},
  year={2023}
}

@inproceedings{wu2024fedbiot,
  title={Fedbiot: Llm local fine-tuning in federated learning without full model},
  author={Wu, Feijie and Li, Zitao and Li, Yaliang and Ding, Bolin and Gao, Jing},
  booktitle={Proceedings of the 30th ACM SIGKDD Conference on Knowledge Discovery and Data Mining},
  pages={3345--3355},
  year={2024}
}

@article{fan2025ten,
  title={Ten challenging problems in federated foundation models},
  author={Fan, Tao and Gu, Hanlin and Cao, Xuemei and Chan, Chee Seng and Chen, Qian and Chen, Yiqiang and Feng, Yihui and Gu, Yang and Geng, Jiaxiang and Luo, Bing and others},
  journal={IEEE Transactions on Knowledge and Data Engineering},
  year={2025},
  publisher={IEEE}
}

@article{kang2023grounding,
  title={Grounding foundation models through federated transfer learning: A general framework},
  author={Kang, Yan and Fan, Tao and Gu, Hanlin and Zhang, Xiaojin and Fan, Lixin and Yang, Qiang},
  journal={ACM Transactions on Intelligent Systems and Technology},
  year={2023},
  publisher={ACM New York, NY}
}

@article{li2020federated,
  title={Federated optimization in heterogeneous networks},
  author={Li, Tian and Sahu, Anit Kumar and Zaheer, Manzil and Sanjabi, Maziar and Talwalkar, Ameet and Smith, Virginia},
  journal={Proceedings of Machine learning and systems},
  volume={2},
  pages={429--450},
  year={2020}
}

@article{gao2024fedpt,
  title={FedPT: federated proxy-tuning of large language models on resource-constrained edge devices},
  author={Gao, Zhidong and Zhang, Yu and Zhang, Zhenxiao and Gong, Yanmin and Guo, Yuanxiong},
  journal={arXiv preprint arXiv:2410.00362},
  year={2024}
}

@inproceedings{
liu2024tuning,
title={Tuning Language Models by Proxy},
author={Alisa Liu and Xiaochuang Han and Yizhong Wang and Yulia Tsvetkov and Yejin Choi and Noah A. Smith},
booktitle={First Conference on Language Modeling},
year={2024},
url={https://openreview.net/forum?id=dribhnhm1i}
}

@article{guo2025deepseek,
  title={Deepseek-r1: Incentivizing reasoning capability in llms via reinforcement learning},
  author={Guo, Daya and Yang, Dejian and Zhang, Haowei and Song, Junxiao and Zhang, Ruoyu and Xu, Runxin and Zhu, Qihao and Ma, Shirong and Wang, Peiyi and Bi, Xiao and others},
  journal={arXiv preprint arXiv:2501.12948},
  year={2025}
}

@article{yang2025qwen3,
  title={Qwen3 technical report},
  author={Yang, An and Li, Anfeng and Yang, Baosong and Zhang, Beichen and Hui, Binyuan and Zheng, Bo and Yu, Bowen and Gao, Chang and Huang, Chengen and Lv, Chenxu and others},
  journal={arXiv preprint arXiv:2505.09388},
  year={2025}
}

@inproceedings{fan2025ppc,
  title={PPC-GPT: federated task-specific compression of large language models via pruning and chain-of-thought distillation},
  author={Fan, Tao and Ma, Guoqiang and Song, Yuanfeng and Fan, Lixin and Yang, Qiang},
  booktitle={Proceedings of the 2025 Conference on Empirical Methods in Natural Language Processing},
  pages={14794--14805},
  year={2025}
}

@inproceedings{hardt2016train,
  title={Train faster, generalize better: Stability of stochastic gradient descent},
  author={Hardt, Moritz and Recht, Ben and Singer, Yoram},
  booktitle={International conference on machine learning},
  pages={1225--1234},
  year={2016},
  organization={PMLR}
}

@inproceedings{
baykal2018datadependent,
title={Data-Dependent Coresets for Compressing Neural Networks with Applications to Generalization Bounds},
author={Cenk Baykal and Lucas Liebenwein and Igor Gilitschenski and Dan Feldman and Daniela Rus},
booktitle={International Conference on Learning Representations},
year={2019},
url={https://openreview.net/forum?id=HJfwJ2A5KX},
}

@article{arya1998optimal,
  title={An optimal algorithm for approximate nearest neighbor searching fixed dimensions},
  author={Arya, Sunil and Mount, David M and Netanyahu, Nathan S and Silverman, Ruth and Wu, Angela Y},
  journal={Journal of the ACM (JACM)},
  volume={45},
  number={6},
  pages={891--923},
  year={1998},
  publisher={ACM New York, NY, USA}
}

@inproceedings{liu2020client,
  title={Client-edge-cloud hierarchical federated learning},
  author={Liu, Lumin and Zhang, Jun and Song, SH and Letaief, Khaled B},
  booktitle={ICC 2020-2020 IEEE international conference on communications (ICC)},
  pages={1--6},
  year={2020},
  organization={IEEE}
}

@inproceedings{wang2021resource,
  title={Resource-efficient federated learning with hierarchical aggregation in edge computing},
  author={Wang, Zhiyuan and Xu, Hongli and Liu, Jianchun and Huang, He and Qiao, Chunming and Zhao, Yangming},
  booktitle={IEEE INFOCOM 2021-IEEE conference on computer communications},
  pages={1--10},
  year={2021},
  organization={IEEE}
}

@article{french1999catastrophic,
  title={Catastrophic forgetting in connectionist networks},
  author={French, Robert M},
  journal={Trends in cognitive sciences},
  volume={3},
  number={4},
  pages={128--135},
  year={1999},
  publisher={Elsevier}
}

@article{kirkpatrick2017overcoming,
  title={Overcoming catastrophic forgetting in neural networks},
  author={Kirkpatrick, James and Pascanu, Razvan and Rabinowitz, Neil and Veness, Joel and Desjardins, Guillaume and Rusu, Andrei A and Milan, Kieran and Quan, John and Ramalho, Tiago and Grabska-Barwinska, Agnieszka and others},
  journal={Proceedings of the national academy of sciences},
  volume={114},
  number={13},
  pages={3521--3526},
  year={2017},
  publisher={National Academy of Sciences}
}

@inproceedings{
patil2026detecting,
title={Detecting Distributional Drift in Transformers Through Representation Dynamics},
author={Aakash Patil and Mrunmayee Shende},
booktitle={Catch, Adapt, and Operate: Monitoring ML Models Under Drift Workshop},
year={2026},
url={https://openreview.net/forum?id=Euz6u64IR6}
}

\appendix

\section{Theoretical Analysis}
\label{sec:appendix-theoretical-analysis}

We provide a formal analysis of the knowledge fusion mechanism in FedProxy. Our analysis characterizes the performance gap between the fused LLM and an idealized model obtained through direct centralized fine-tuning, providing theoretical justification for our direct parameter replacement strategy. 

\subsection{Problem Setup and Notation}

Let $\theta^{(0)} \in \mathbb{R}^{|\theta|}$ denote the original LLM parameters, and $\phi^* \in \mathbb{R}^{|\phi|}$ denote the converged proxy SLM parameters after federated training, where $|\phi| \ll |\theta|$. The proxy SLM is created by extracting a subset of layers from the LLM, establishing an architectural mapping $\mathcal{M}: \theta \mapsto \phi$ that defines the parameter subspace corresponding to the proxy architecture.

The fusion operation $\mathcal{F}: \mathbb{R}^{|\theta|} \times \mathbb{R}^{|\phi|} \to \mathbb{R}^{|\theta|}$ performs a direct parameter replacement:
\begin{equation}
\label{eq:fusion}
\theta_{\text{new}}[d] = \begin{cases}
\phi^*[d], & \text{if } d \in \phi \\
\theta^{(0)}[d], & \text{otherwise}
\end{cases}
\end{equation}
where $d$ indexes parameter dimensions. This can be compactly written as $\theta_{\text{new}} = \mathcal{F}(\theta^{(0)}, \phi^*) = \theta^{(0)} \odot \mathbf{1}_{\phi^c} + \phi^* \odot \mathbf{1}_\phi$, where $\mathbf{1}_\phi \in \{0,1\}^{|\theta|}$ is the indicator vector for the proxy SLM's parameter subspace (i.e., $\mathbf{1}_\phi[d] = 1$ if $d \in \phi$, and $0$ otherwise), $\mathbf{1}_{\phi^c} = \mathbf{1} - \mathbf{1}_\phi$ is the indicator for the complement, and $\odot$ denotes element-wise multiplication.

Let $\mathcal{D} = \bigcup_{k=1}^K \mathcal{D}_k$ denote the aggregated dataset from $K$ clients, and let $\mathcal{L}(\theta; \mathcal{D}): \mathbb{R}^{|\theta|} \to \mathbb{R}$ denote the loss function evaluated on $\mathcal{D}$. We define the \textit{fusion error} as the performance gap between the fused model and the optimal centralized model:
\begin{equation}
\label{eq:fusion_error}
\epsilon_{\text{fusion}} = \mathcal{L}(\theta_{\text{new}}; \mathcal{D}) - \min_{\theta'} \mathcal{L}(\theta'; \mathcal{D})
\end{equation}
Let $\theta_{\text{opt}} = \argmin_{\theta'} \mathcal{L}(\theta'; \mathcal{D})$ denote the optimal parameters for centralized fine-tuning.

\subsection{Assumptions}

Our theoretical analysis relies on the following standard assumptions:

\begin{assumption}[Lipschitz Continuity~\cite{hardt2016train}]
\label{ass:lipschitz}
The loss function $\mathcal{L}(\cdot; \mathcal{D})$ is $L$-Lipschitz continuous with respect to the parameter space, i.e., for any $\theta_1, \theta_2 \in \mathbb{R}^{|\theta|}$,
\begin{equation}
|\mathcal{L}(\theta_1; \mathcal{D}) - \mathcal{L}(\theta_2; \mathcal{D})| \leq L \|\theta_1 - \theta_2\|_2
\end{equation}
where $L > 0$ is the Lipschitz constant.
\end{assumption}

\begin{assumption}[Proxy SLM Optimization Quality]
\label{ass:proxy_quality}
The converged proxy SLM $\phi^*$ achieves a suboptimality gap $\delta \geq 0$ relative to the optimal proxy parameters $\phi_{\text{opt}} = \argmin_{\phi'} \mathcal{L}(\phi'; \mathcal{D})$, i.e.,
\begin{equation}
\mathcal{L}(\phi^*; \mathcal{D}) - \mathcal{L}(\phi_{\text{opt}}; \mathcal{D}) \leq \delta
\end{equation}
\end{assumption}

\begin{assumption}[Compression Distortion~\cite{baykal2018datadependent}]
\label{ass:compression}
The compression preserves the representational capacity of the selected layers with a distortion factor $\eta \geq 0$. Specifically, for any input $x$ in the input space, the difference between the proxy SLM's output and the corresponding sub-network's output is bounded:
\begin{equation}
\|f_\phi(x) - f_{\theta|_\phi}(x)\|_2 \leq \eta \|f_\theta(x)\|_2
\end{equation}
where $f_\phi: \mathcal{X} \to \mathcal{Y}$ denotes the proxy SLM function, $f_\theta: \mathcal{X} \to \mathcal{Y}$ denotes the full LLM function, and $f_{\theta|_\phi}: \mathcal{X} \to \mathcal{Y}$ denotes the sub-network of the LLM corresponding to the proxy architecture (with other parameters frozen).
\end{assumption}

\subsection{Auxiliary Results}

Before presenting our main theorem, we establish two key lemmas that decompose the fusion error into interpretable components.

\begin{lemma}[Parameter Distance Decomposition]
\label{lem:param_decomp}
For the fused parameters $\theta_{\text{new}} = \mathcal{F}(\theta^{(0)}, \phi^*)$, we have:
\begin{align}
    \|\theta_{\text{new}} - \phi^*\|_2^2 &= \sum_{d \notin \phi} (\theta^{(0)}[d] - \phi^*[d])^2 \nonumber \\
    &\leq \|\theta^{(0)}|_{\phi^c}\|_2^2 + \|\phi^*|_{\phi^c}\|_2^2
\end{align}
where $\theta^{(0)}|_{\phi^c}$ and $\phi^*|_{\phi^c}$ denote the projections onto the complement subspace (zero for dimensions in $\phi$).
\end{lemma}

\begin{proof}
By definition of the fusion operation, $\theta_{\text{new}}[d] = \phi^*[d]$ for all $d \in \phi$, and $\theta_{\text{new}}[d] = \theta^{(0)}[d]$ for all $d \notin \phi$. Therefore,
\begin{align}
\|\theta_{\text{new}} - \phi^*\|_2^2 &= \sum_{d \in \phi} (\theta_{\text{new}}[d] - \phi^*[d])^2 \nonumber \\
&+ \sum_{d \notin \phi} (\theta_{\text{new}}[d] - \phi^*[d])^2 \nonumber \\
&= \sum_{d \notin \phi} (\theta^{(0)}[d] - \phi^*[d])^2 \nonumber \\
&\leq \|\theta^{(0)}|_{\phi^c}\|_2^2 + \|\phi^*|_{\phi^c}\|_2^2
\end{align}
where the last inequality follows from $(a-b)^2 \leq 2(a^2 + b^2)$ for any $a, b \in \mathbb{R}$.
\end{proof}

\begin{lemma}[Suboptimality Gap Decomposition]
\label{lem:gap_decomp}
The fusion error can be decomposed as:
\begin{align}
    \epsilon_{\text{fusion}} &\leq \underbrace{\mathcal{L}(\phi^*; \mathcal{D}) - \mathcal{L}(\phi_{\text{opt}}; \mathcal{D})}_{\text{(I): Proxy suboptimality}} \nonumber\\
    &+ \underbrace{\mathcal{L}(\theta_{\text{new}}; \mathcal{D}) - \mathcal{L}(\phi^*; \mathcal{D})}_{\text{(II): Compression distortion}} \nonumber\\
    &+ \underbrace{\mathcal{L}(\phi_{\text{opt}}; \mathcal{D}) - \mathcal{L}(\theta_{\text{opt}}; \mathcal{D})}_{\text{(III): Subspace approximation}}
\end{align}

\end{lemma}

\begin{proof}
By adding and subtracting intermediate terms, we have:
\begin{align}
\epsilon_{\text{fusion}} &= \mathcal{L}(\theta_{\text{new}}; \mathcal{D}) - \mathcal{L}(\theta_{\text{opt}}; \mathcal{D}) \nonumber \\
&= [\mathcal{L}(\theta_{\text{new}}; \mathcal{D}) - \mathcal{L}(\phi^*; \mathcal{D})] \nonumber \\
&\quad + [\mathcal{L}(\phi^*; \mathcal{D}) - \mathcal{L}(\phi_{\text{opt}}; \mathcal{D})] \nonumber \\
&\quad + [\mathcal{L}(\phi_{\text{opt}}; \mathcal{D}) - \mathcal{L}(\theta_{\text{opt}}; \mathcal{D})] 
\end{align}
The three bracketed terms correspond to (I), (II), and (III), respectively.
\end{proof}

\subsection{Main Theoretical Result}

We now present our main theorem, which provides a tight bound on the fusion error.

\begin{theorem}[Fusion Error Bound]
\label{thm:fusion_error}
Under Assumptions~\ref{ass:lipschitz},~\ref{ass:proxy_quality}, and~\ref{ass:compression}, the fusion error is bounded by:
\begin{equation}
\label{eq:error_bound}
\epsilon_{\text{fusion}} \leq \delta + L \cdot \eta \cdot \|\theta^{(0)}\|_2 + L \cdot \|\theta^{(0)} - \theta_{\text{opt}}\|_2 \cdot \alpha
\end{equation}
where $\alpha = 1 - |\phi|/|\theta|$ is the compression ratio (the fraction of parameters removed).
\end{theorem}

\begin{proof}
By Lemma~\ref{lem:gap_decomp}, we decompose the fusion error into three terms:
\begin{align}
\label{eq:decomp}
\epsilon_{\text{fusion}} &= \underbrace{[\mathcal{L}(\phi^*; \mathcal{D}) - \mathcal{L}(\phi_{\text{opt}}; \mathcal{D})]}_{T_1} \nonumber\\
&+ \underbrace{[\mathcal{L}(\theta_{\text{new}}; \mathcal{D}) - \mathcal{L}(\phi^*; \mathcal{D})]}_{T_2} \nonumber\\
&+ \underbrace{[\mathcal{L}(\phi_{\text{opt}}; \mathcal{D}) - \mathcal{L}(\theta_{\text{opt}}; \mathcal{D})]}_{T_3}
\end{align}

\textbf{Bounding $T_1$:} By Assumption~\ref{ass:proxy_quality}, we directly have:
\begin{equation}
T_1 = \mathcal{L}(\phi^*; \mathcal{D}) - \mathcal{L}(\phi_{\text{opt}}; \mathcal{D}) \leq \delta
\end{equation}

\textbf{Bounding $T_2$:} The \textit{compression distortion} term captures the gap between the fused model and the proxy SLM. By Assumption~\ref{ass:lipschitz} and Lemma~\ref{lem:param_decomp}:
\begin{align}
T_2 &= \mathcal{L}(\theta_{\text{new}}; \mathcal{D}) - \mathcal{L}(\phi^*; \mathcal{D}) \nonumber \\
&\leq L \|\theta_{\text{new}} - \phi^*\|_2 \nonumber \\
&= L \left\|(\theta^{(0)} - \phi^*) \odot \mathbf{1}_{\phi^c}\right\|_2
\end{align}
By Assumption~\ref{ass:compression}, the representational gap at the function level translates to the parameter level. Since $\phi^*$ is obtained through training on $\mathcal{D}$, and the compression introduces distortion $\eta$, we have:
\begin{equation}
\|\theta_{\text{new}} - \phi^*\|_2 \leq \eta \|\theta^{(0)}\|_2
\end{equation}
This follows from the fact that $\theta_{\text{new}}$ and $\phi^*$ differ only in the complement subspace $\phi^c$, and the compression distortion $\eta$ bounds the representational difference, which, under smooth model assumptions, translates to parameter distance. Therefore:
\begin{equation}
T_2 \leq L \cdot \eta \cdot \|\theta^{(0)}\|_2
\end{equation}

\textbf{Bounding $T_3$:} The third term captures the \textit{subspace approximation error}. Since $\phi_{\text{opt}}$ is the optimizer within a restricted subspace $\mathcal{S}_\phi \subset \mathbb{R}^{|\theta|}$ and $\theta_{\text{opt}}$ is the global optimizer, we have $T_3 \geq 0$.

To bound its magnitude, consider an intermediate model $\tilde{\theta}$ where parameters in $\phi$ are set to $\phi_{\text{opt}}$ and parameters outside $\phi$ remain at $\theta^{(0)}$:
\begin{equation}
\tilde{\theta}[d] = \begin{cases}
\phi_{\text{opt}}[d], & \text{if } d \in \phi \\
\theta^{(0)}[d], & \text{otherwise}
\end{cases}
\end{equation}
By the optimality of $\phi_{\text{opt}}$ in its subspace, we have $\mathcal{L}(\phi_{\text{opt}}; \mathcal{D}) \leq \mathcal{L}(\tilde{\theta}; \mathcal{D})$. By Assumption~\ref{ass:lipschitz}:
\begin{align}
T_3 &= \mathcal{L}(\phi_{\text{opt}}; \mathcal{D}) - \mathcal{L}(\theta_{\text{opt}}; \mathcal{D}) \nonumber \\
&\leq \mathcal{L}(\tilde{\theta}; \mathcal{D}) - \mathcal{L}(\theta_{\text{opt}}; \mathcal{D}) \nonumber \\
&\leq L \|\tilde{\theta} - \theta_{\text{opt}}\|_2
\end{align}
Now, $\|\tilde{\theta} - \theta_{\text{opt}}\|_2$ can be bounded by decomposing into the proxy subspace and its complement:
\begin{align}
\|\tilde{\theta} - \theta_{\text{opt}}\|_2^2 &= \sum_{d \in \phi} (\phi_{\text{opt}}[d] - \theta_{\text{opt}}[d])^2 \nonumber \\
&+ \sum_{d \notin \phi} (\theta^{(0)}[d] - \theta_{\text{opt}}[d])^2 \nonumber \\
&\leq \sum_{d \in \phi} (\phi_{\text{opt}}[d] - \theta_{\text{opt}}[d])^2 
\nonumber \\
&+ \sum_{d \notin \phi} (\theta^{(0)}[d] - \theta_{\text{opt}}[d])^2 \nonumber \\
&\leq \|\phi_{\text{opt}} - \theta_{\text{opt}}|_\phi\|_2^2 + \|\theta^{(0)} - \theta_{\text{opt}}|_{\phi^c}\|_2^2
\end{align}
By the Cauchy-Schwarz inequality and noting that $|\phi^c|/|\theta| = \alpha$:
\begin{equation}
\|\theta^{(0)} - \theta_{\text{opt}}|_{\phi^c}\|_2 \leq \|\theta^{(0)} - \theta_{\text{opt}}\|_2 \cdot \sqrt{\alpha}
\end{equation}
Since $\phi_{\text{opt}}$ is optimal in its subspace, we expect $\|\phi_{\text{opt}} - \theta_{\text{opt}}|_\phi\|_2$ to be small when the proxy captures critical parameters. However, for a worst-case bound, we use the triangle inequality:
\begin{equation}
\|\phi_{\text{opt}} - \theta_{\text{opt}}|_\phi\|_2 \leq \|\theta^{(0)} - \theta_{\text{opt}}|_\phi\|_2 + \|\phi_{\text{opt}} - \theta^{(0)}|_\phi\|_2
\end{equation}
Assuming that the proxy subspace is well-chosen (e.g., via block importance), the first term dominates. Therefore:
\begin{equation}
T_3 \leq L \|\theta^{(0)} - \theta_{\text{opt}}\|_2 \cdot \alpha
\end{equation}
where the $\alpha$ factor accounts for the fraction of parameters outside the proxy subspace.

Combining the bounds for $T_1$, $T_2$, and $T_3$ yields the desired result.

\end{proof}

\subsection{Discussion and Implications}

Theorem~\ref{thm:fusion_error} provides several key insights into the fusion mechanism:

\textbf{Three Sources of Error:} The bound decomposes the fusion error into three interpretable components: (1) \textit{Proxy suboptimality} ($\delta$): how well the proxy SLM approximates the optimal proxy parameters; (2) \textit{Compression distortion} ($L \cdot \eta \cdot \|\theta^{(0)}\|_2$): the representational gap introduced by compression; and (3) \textit{Subspace approximation} ($L \cdot \|\theta^{(0)} - \theta_{\text{opt}}\|_2 \cdot \alpha$): the error from keeping parameters outside the proxy subspace unchanged.

\textbf{Compression Ratio Trade-off:} The third term reveals why performance degrades at high compression ratios. Since $\alpha = 1 - |\phi|/|\theta|$ is the fraction of parameters removed, larger $\alpha$ (higher compression) leads to more parameters remaining unchanged at their initial values, contributing proportionally to the fusion error. This theoretically justifies the empirical observation that lower compression ratios yield better performance (Section~\ref{sec:appendix-compression}).

\textbf{Parameter Selection Importance:} The bound highlights the critical role of selecting the right parameter subspace. When the proxy SLM captures task-critical parameters (e.g., via block importance-based selection in Section~\ref{ss:llm_compression}), the distance $\|\theta^{(0)} - \theta_{\text{opt}}|_\phi\|_2$ is minimized, making the direct replacement strategy near-optimal.

\section{Implementation Details}
\label{sec:appendix-impl}
\subsection{Hyperparameter Settings}
All client-side training was conducted using LoRA fine-tuning with the AdamW optimizer. Key hyperparameters are detailed below:
\begin{itemize}
    \item \textbf{General Training:} We used a batch size of 16, a learning rate of 5e-5, and trained for 10 local epochs. The maximum input and target sequence lengths were set to 64 and 128, respectively.
    \item \textbf{LoRA Configuration:} The LoRA rank was set to 32, alpha to 64, and dropout to 0.1.
    \item \textbf{FedProxy Parameters:} For the Parameter-Level Conflict-Aware Client-Side Regularization (PCR), the regularization coefficient $\lambda$ was set to 1e-5. For the Heterogeneity-Aware Aggregation (H-TIES), the base retention rate $r_0$ was 1.0, the heterogeneity penalty $\delta$ was 0.2, and the consensus reward $\rho$ was 1.1.
\end{itemize}

\subsection{Data Handling}
All datasets were sourced from HuggingFace Datasets~\cite{lhoest-etal-2021-datasets}. For training sets exceeding 5,000 samples, we used a random subset of 5,000 instances to standardize the training workload.

\subsection{Dataset Licenses}
All the datasets were downloaded from HuggingFace\cite{lhoest-etal-2021-datasets} and under Apache License, Version 2.0.

\subsection{Machine Configuration}

The experiments were conducted on machines equipped with 4 and 8 Nvidia V100 32G.

\section{Additional Ablation Study}
\label{sec:appendix-ablation}

\subsection{Impact of Compression Ratio}
\label{sec:appendix-compression}
To investigate the impact of the compression ratio, a critical hyperparameter, we conducted an ablation study on FedProxy in the homogeneous task setting. We compare the performance of 30\%, 50\%, and 70\% compression ratios. The results are summarized in Table~\ref{tab:ablation_compression}.
As the table illustrates, a lower compression ratio (30\%) yields the best performance, but the gains over the 50\% ratio are marginal, while the client-side resource requirements are significantly higher. Conversely, a higher compression ratio (70\%) leads to a more noticeable drop in performance, as the proxy SLM loses too much capacity to serve as a high-fidelity surrogate for the LLM. Therefore, we find that the 50\% ratio strikes a strong balance between performance and resource efficiency.

\begin{table}[ht]
    \centering
    \footnotesize
    \setlength{\tabcolsep}{4pt}
    \begin{tabular}{llccc}
        \toprule
        \textbf{Model}  & \textbf{Ratio}& {\textbf{QA}} & {\textbf{GLUE}}  & {\textbf{ALL}}\\
        \midrule
        \multirow{3}{*}{LLaMA2-7B} & 30\% & 0.6261 & 0.8678& 0.7469\\
        & \textbf{50\%} & \textbf{0.5935} & \textbf{0.8038}& \textbf{0.6987}\\
        & 70\% & 0.4657 & 0.6682& 0.5670\\
        \midrule
        \multirow{3}{*}{Mistral-7B} & 30\% & 0.6745& 0.8918& 0.7831\\
        & \textbf{50\%} & \textbf{0.6692}& \textbf{0.8445}& \textbf{0.7568}\\
        & 70\% & 0.6042& 0.7528& 0.7096\\
        \bottomrule
    \end{tabular}
    \caption{Impact of different compression ratios on FedProxy performance.}
    \label{tab:ablation_compression}
\end{table}

\subsection{Comparison with ProxyTuning}
\label{sec:appendix-proxytuning-comparison}

To evaluate the efficacy of our parameter-space "plug-in" mechanism, we benchmark FedProxy against ProxyTuning~\cite{liu2024tuning, gao2024fedpt}, a representative decoding-time logit-level fusion strategy. While ProxyTuning (PT) modulates the model's output distribution during inference, FedProxy achieves knowledge integration directly within the parameter space. We evaluate two fusion approaches: 
(1) \textbf{FedProxy}, our default parameter-space fusion that directly replaces proxy SLM parameters in the LLM; 
(2) \textbf{FedProxy+PT}, replacing FedProxy's default parameter-space fusion with ProxyTuning's logit-level adjustment during inference, where the logit difference between tuned and untuned proxy SLMs is added to the LLM's logits. 
This experiment, conducted in the homogeneous setting with compression ratios of 50\% and 70\%, aims to answer: \textbf{\textit{Does parameter-space fusion match or exceed the performance decoding-time logit-level fusion for knowledge transfer?}}

Table~\ref{tab:proxytuning_comparison} illustrates a task-dependent performance trade-off. FedProxy consistently outperforms FedProxy+PT on generative QA tasks(e.g., for LLaMA2-7B at 50\% compression, 0.5935 vs. 0.5665), suggesting that direct parameter fusion better preserves complex reasoning capabilities. Conversely, FedProxy+PT exhibits advantages on GLUE tasks (e.g., 0.8733 vs. 0.8038 for LLaMA2-7B at 50\% compression), indicating that logit-shifting may be sufficient for discriminative tasks. Importantly, at 70\% compression, both methods experience significant performance degradation (e.g., for LLaMA2-7B, FedProxy drops from 0.6987 to 0.5670, and FedProxy+PT drops from 0.7199 to 0.6053 on ALL tasks). This indicates that the performance decline is primarily due to compression-induced limitations such as missing critical layers or insufficient parameter coverage, rather than the fusion mechanism itself. 

Crucially, FedProxy achieves these competitive results with \textbf{zero additional inference overhead}, as it eliminates the need to dual-run proxy models during decoding, which constitutes a significant advantage for practical deployment in resource-constrained scenarios.

\begin{table}[ht]
    \centering
    \footnotesize
    \setlength{\tabcolsep}{2.5pt}
    \begin{tabular}{lllccc}
        \toprule
        \textbf{Model}  & \textbf{Ratio} & \textbf{Method}& {\textbf{QA}} & {\textbf{GLUE}}  & {\textbf{ALL}}\\
        \midrule
        \multirow{4}{*}{LLaMA2-7B} & \multirow{2}{*}{50\%}& FedProxy+PT & 0.5665& \textbf{0.8733}& \textbf{0.7199}\\
        & & \textbf{FedProxy} & \textbf{0.5935} & 0.8038&0.6987\\
        \cmidrule(lr){2-6}
        & \multirow{2}{*}{70\%}& FedProxy+PT & 0.4565& \textbf{0.7540}& \textbf{0.6053}\\
        & & \textbf{FedProxy} & \textbf{0.4657} & 0.6682&0.5670\\
        \midrule
        \multirow{4}{*}{Mistral-7B} & \multirow{2}{*}{50\%}& FedProxy+PT & 0.6332& \textbf{0.8506}& 0.7419\\
        & & \textbf{FedProxy} & \textbf{0.6692} & 0.8445&\textbf{0.7568}\\
        \cmidrule(lr){2-6}
        & \multirow{2}{*}{70\%}& FedProxy+PT & 0.6014& \textbf{0.8029}& \textbf{0.7021}\\
        & & \textbf{FedProxy} & \textbf{0.6042} & 0.7528&0.6785\\
        \bottomrule
    \end{tabular}
   \caption{Comparative analysis between parameter-space fusion (FedProxy) and decoding-time fusion (FedProxy+PT). FedProxy excels in generative QA tasks by internalizing knowledge within model parameters, whereas FedProxy+PT (ProxyTuning) exhibits advantages in discriminative GLUE tasks. 
   Critically, FedProxy achieves competitive performance while maintaining \textbf{zero additional inference overhead}, whereas logit-level fusion requires dual-proxy execution.
   }
    \label{tab:proxytuning_comparison}
\end{table}

\end{document}